\documentclass{article}

    \PassOptionsToPackage{numbers, sort}{natbib}
\usepackage{enumitem,kantlipsum}

    \usepackage[final]{neurips_2021_nofoot}

\usepackage[utf8]{inputenc} % allow utf-8 input
\usepackage[T1]{fontenc}    % use 8-bit T1 fonts
\usepackage{hyperref}       % hyperlinks
\usepackage{url}            % simple URL typesetting
\usepackage{booktabs}       % professional-quality tables
\usepackage{amsfonts}       % blackboard math symbols
\usepackage{nicefrac}       % compact symbols for 1/2, etc.
\usepackage{microtype}      % microtypography
\usepackage[dvipsnames]{xcolor}
\usepackage{wrapfig}
\newcommand{\define}[1]{{\bf \boldmath{#1}}}

\usepackage{amsmath,amsfonts,bm,amsthm,amssymb}

\newtheorem{theorem}{Theorem}
\newtheorem{proposition}[theorem]{Proposition}
\newtheorem{definition}[theorem]{Definition}

\def\eqref#1{equation~\ref{#1}}
\def\1{\bm{1}}

\def\eps{{\epsilon}}

\def\rmB{{\mathbf{B}}}

\def\rmP{{\mathbf{P}}}

\def\rmT{{\mathbf{T}}}

\def\rmX{{\mathbf{X}}}

\def\mA{\bm{A}}
\def\mB{\bm{B}}

\def\mI{{\bm{I}}}

\def\mL{{\bm{L}}}

\def\mP{{\bm{P}}}

\def\mT{{\bm{T}}}

\def\mX{\bm{X}}

\DeclareMathAlphabet{\mathsfit}{\encodingdefault}{\sfdefault}{m}{sl}
\SetMathAlphabet{\mathsfit}{bold}{\encodingdefault}{\sfdefault}{bx}{n}

\def\gB{{\mathcal{B}}}
\def\gC{{\mathcal{C}}}

\def\gF{{\mathcal{F}}}
\def\gG{{\mathcal{G}}}

\def\gK{{\mathcal{K}}}

\def\gN{{\mathcal{N}}}
\def\gO{{\mathcal{O}}}

\def\gX{{\mathcal{X}}}

\def\sC{{\mathbb{C}}}

\def\sN{{\mathbb{N}}}

\def\sR{{\mathbb{R}}}

\newcommand{\nup}{\gN_{\uparrow}}
\newcommand{\ndown}{\gN_{\downarrow}}

\newcommand{\da}{\downarrow}
\newcommand{\ua}{\uparrow}

\usepackage{todonotes}
\usepackage{multirow}
\usepackage{multicol}
\usepackage{caption}
\usepackage{subcaption}
\usepackage{wrapfig}
\usepackage{mathtools}
\usepackage{scalerel}
\usepackage{amssymb}

\newtheorem{lemma}[theorem]{Lemma}
\newtheorem{corollary}[theorem]{Corollary}

\newtheorem{example}[theorem]{Example}
\renewcommand{\eqref}[1]{(\ref{#1})}

\newcommand*{\ldblbrace}{\{\mskip-5mu\{}
\newcommand*{\rdblbrace}{\}\mskip-5mu\}}

\def\mcirc{\mathbin{\scalerel*{\bigcirc}{t}}}

\newcommand{\FF}{\mathcal{F}}
\newcommand{\Hilb}{\mathrm{Hilb}}

\newcommand{\tte}{\texttt{e}}

\newcommand{\first}[1]{\textbf{\textcolor{red}{#1}}}
\newcommand{\second}[1]{\textbf{\textcolor{violet}{#1}}}
\newcommand{\third}[1]{\textbf{\textcolor{black}{#1}}}

\title{Weisfeiler and Lehman Go Cellular: CW Networks}

\author{%
  Cristian Bodnar\thanks{Authors contributed equally.}  \\
  University of Cambridge\\
  \texttt{cb2015@cam.ac.uk} \\
   \And
   Fabrizio Frasca$^*$ \\
   Imperial College London \& Twitter \\
   \texttt{ffrasca@twitter.com} \\
   \AND
   Nina Otter \\
   UCLA \\
   \texttt{otter@math.ucla.edu} \\
   \And
   Yu Guang Wang \\
   MPI-MIS, SJTU \& UNSW \\
   \texttt{yuguang.wang@unsw.edu.au} \\
   \AND
   Pietro Li\`{o} \\
   University of Cambridge \\
   \texttt{pl219@cam.ac.uk} \\
   \And
   Guido Mont\'{u}far \\
   MPI-MIS \& UCLA \\
   \texttt{montufar@math.ucla.edu} \\
   \And
   Michael Bronstein \\
   Imperial College London \& Twitter \\
   \texttt{mbronstein@twitter.com} \\
}

\begin{document}

\maketitle

\begin{abstract}
Graph Neural Networks (GNNs) are limited in their expressive power, struggle with long-range interactions and lack a principled way to model higher-order structures.
These problems can be attributed to the strong coupling between the computational graph and the input graph structure.
The recently proposed Message Passing Simplicial Networks naturally decouple these elements by performing message passing on the clique complex of the graph. Nevertheless, these models can be severely constrained by the rigid combinatorial structure of Simplicial Complexes (SCs). 
In this work, we extend recent theoretical results on SCs to regular Cell Complexes, topological objects that flexibly subsume SCs and graphs. 
We show that this generalisation provides a powerful set of graph ``lifting'' transformations, each leading to a unique hierarchical message passing procedure. The resulting methods, which we collectively call CW Networks (CWNs), are strictly more powerful than the WL test and not less powerful than the 3-WL test. In particular, we demonstrate the effectiveness of one such scheme, based on rings, when applied to molecular graph problems. The proposed architecture benefits from provably larger expressivity than commonly used GNNs, principled modelling of higher-order signals and from compressing the distances between nodes. We demonstrate that our model achieves state-of-the-art results on a variety of molecular datasets.
\end{abstract}

\section{Introduction}

The operations performed by message passing Graph Neural Networks (GNNs) emulate the structure of the input graph. While this property has clear computational advantages, it brings with it a series of fundamental limitations. As observed by \citet{GIN} and \citet{morris2019weisfeiler} the local neighbourhood aggregations used by GNNs are at most as powerful as the Weisfeiler-Lehman (WL) test \citep{weisfeiler1968reduction} in distinguishing non-isomorphic graphs. Therefore, GNNs fail to detect certain higer-order meso-scale structures such as cliques or (induced) cycles~\citep{ARVIND202042, chen2020can}, which are particularly important in applications dealing with social and biological networks or molecular graphs. At the same time, many such layers have to be stacked to make long-range interactions in the graph possible. Besides the computational burden this incurs, deep GNNs typically come with additional problems such as over-smoothing \citep{li2018deeper} and over-squashing \citep{alon2020bottleneck} of the node representations.  

To address these problems, we propose a novel message passing procedure based on (regular) cell complexes, also known as CW complexes\footnote{We use these terms interchangeably. For the latter, the C stands for ``closure-finite'', and the W for ``weak'' topology. The term was coined by %J. H. C. Whitehead 
\citet{bams/1183513543}.}, topological objects that form the building block of algebraic topology \citep{hatcher_book}. When paired with a theoretically-justified ``lifting'' transformation augmenting the graph with higher-dimensional constructs called ``cells'', our method results in a multi-dimensional and hierarchical message passing procedure over the input graph. Our approach generalises and subsumes the recently proposed Message Passing Simplicial Networks (MPSNs) \citep{bodnar2021weisfeiler}, which operate on simplicial complexes (SCs), topological generalisations of graphs. However, SCs have a rigid combinatorial structure that significantly limits the range of lifting transformations one could use to meaningfully modulate the message passing procedure. In contrast, we show that cell complexes, which in turn generalise simplicial complexes and come with additional flexibility, allow one to construct new and better ways of decoupling the input and computational graphs. 

\paragraph{Main Contributions} To summarise, we propose a message passing scheme operating on regular cell complexes. We call this family of models CW Networks (CWNs) and study their expressive power using a cellular version of the WL test. We show that for an entire class of ``lifting'' transformations CWNs are at least as powerful as the WL test. Furthermore, we prove that for some of the maps in this class, CWNs can be strictly more powerful than WL, Simplicial WL (SWL) and also not less powerful than 3-WL. We also express the fundamental symmetries of these models and show how they can be seen as generalised convolutional operators on cell complexes. Experimentally, we focus our attention on a particular ``lifting'' map based on induced cycles. When applied to molecular graphs, it leads to an intuitive hierarchical message passing procedure involving the atoms, the bonds between them and the chemical rings of the molecules. We demonstrate that this provably powerful approach obtains state-of-the-art results on popular large-scale molecular graph datasets and other related tasks. To the best of our knowledge, this is the first work proposing a cell complex representation for molecules. Our code is available at \url{https://github.com/twitter-research/cwn}. 

\section{Background}

\begin{definition}[\citet{HGh19}]
A \define{regular cell complex} (Figure \ref{fig:cell_def}) is a topological space $X$ together with a partition $\{X_\sigma\}_{\sigma\in P_X}$ of subspaces $X_\sigma$ of $X$ called \define{cells}, and such that
\begin{enumerate}[leftmargin=*,topsep=0pt,itemsep=-0.5ex]
    \item  For each $x\in X$ there exists an open neighbordhood of $x$ that intersects finitely many cells.
\item  For all $\sigma,\tau$ we have that $X_\tau\cap \overline{X_\sigma}\ne \emptyset $ iff $X_\tau\subseteq \overline{X_\sigma}$, where $\overline{X_\sigma}$ is the closure of a cell. 
\item Every cell is homeomorphic to $\mathbb{R}^n$ for some $n$.
\item  ({\em Regularity}) For every $\sigma\in P_X$ there is a homeomorphism $\phi$ of a closed ball in $\mathbb{R}^{n_\sigma}$ to $\overline{X_\sigma}$ such that the restriction of $\phi$ to the  interior of the ball is a homeomorphism onto $X_\sigma$.
\end{enumerate}
\end{definition}

\begin{wrapfigure}{r}{0.35\textwidth}
    \centering
    \vspace{-15pt}
    \begin{subfigure}[b]{1.0\linewidth}
        \centering
        \includegraphics[width=\textwidth]{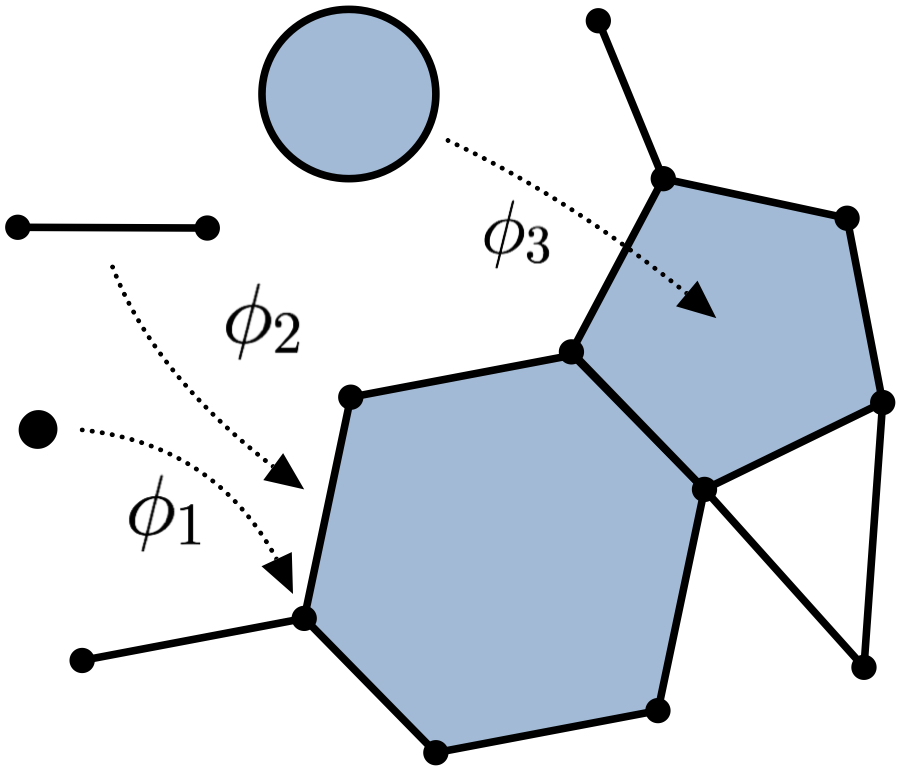}
        % \caption{A sphere}
        \label{fig:y equals x}
     \end{subfigure}\hspace{3mm}
    \vspace{-23pt}
    \caption{A cell complex $X$ and the corresponding homeomorphisms to the closed balls for three cells of different dimensions in the complex.}
    \label{fig:cell_def}
    \vspace{-20pt}
\end{wrapfigure}

% \begin{wrapfigure}{r}{0.27\textwidth}
%     \centering
%     \vspace{-15pt}
%     \begin{subfigure}[b]{0.12\textwidth}
%         \centering
%         \includegraphics[width=\textwidth]{figures/sphere.jpeg}
%         % \caption{A sphere}
%         \label{fig:y equals x}
%      \end{subfigure}\hspace{3mm}
%      \begin{subfigure}[b]{0.12\textwidth}
%         \centering
%         \includegraphics[width=\textwidth]{figures/empty_tetrahderon.jpeg}     
%         % \caption{An empty tetrahedron}
%         \label{fig:y equals x2}
%      \end{subfigure}
%     \vspace{-23pt}
%     \caption{Examples of regular cell complexes: a sphere and an empty tetrahedron. The latter is also a simplicial complex.}
%     \label{fig:1}
%     \vspace{-20pt}

% \end{wrapfigure}

We note that by  condition $(2)$ the indexing set $P_X$ has a poset structure $\tau\leq \sigma \Leftrightarrow X_\tau\subseteq \overline{X_\sigma}$,
while condition $(4)$ guarantees that this poset structure encodes all the topological information about $X$. Thus, we can identify a regular cell complex $X$ with this poset, called \define{face poset} of $X$. We also use $\tau < \sigma$ for the strict version of this partial order.  

Intuitively, one constructs a cell complex through a hierarchical gluing procedure. One starts with a set of vertices (0-cells). Then edges (1-cells) are attached to these by gluing the endpoints of closed line segments to them. We have now only described a (multi) graph. However, one can generalise this even further by taking a two-dimensional closed disk and glue its boundary (i.e. a circle) to any simple cycle in the (multi) graph previously built as in Figure \ref{fig:gluing}. While we are generally not concerned with dimensions above two, this can be further generalised by gluing the boundary of $n$-dimensional balls to certain $(n-1)-$cells in the complex. 

\begin{figure}
    \centering
     \begin{minipage}{.63\textwidth}
         \includegraphics[width=0.48\linewidth]{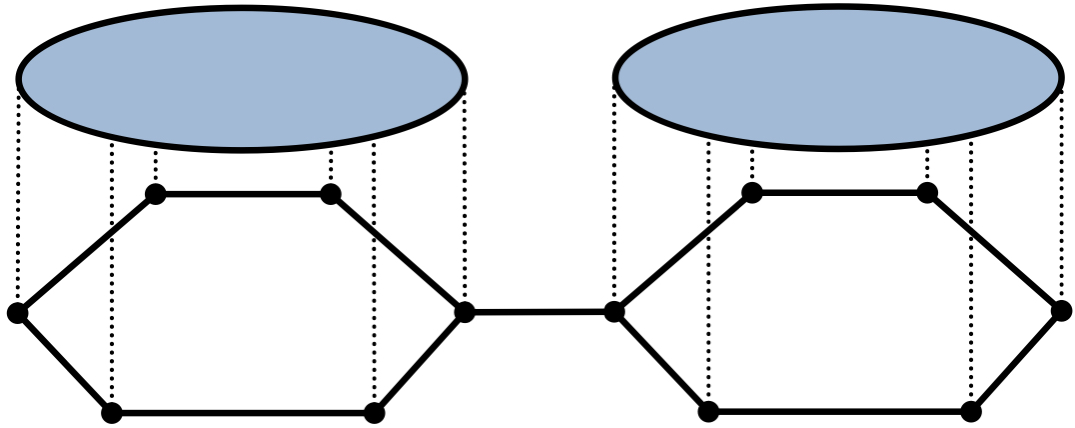}
         \hfill
         \includegraphics[width=0.48\linewidth]{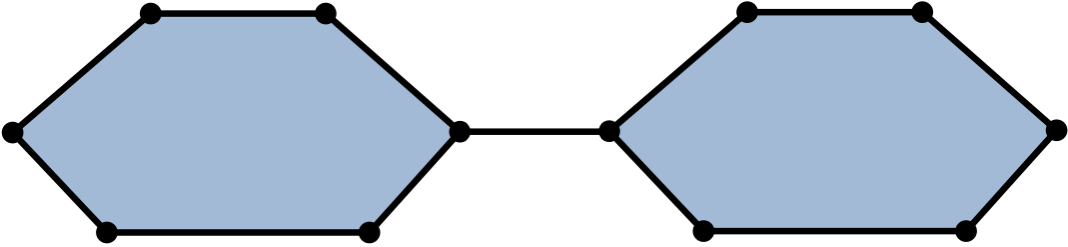}
         \caption{Closed two-dimensional disks are glued to the boundary of the rings present in the graph (left). The result is a 2D regular cell complex (right).}
        \label{fig:gluing}
     \end{minipage}
     \hfill
     \begin{minipage}{0.33\textwidth}
         \centering
         \includegraphics[width=0.35\linewidth]{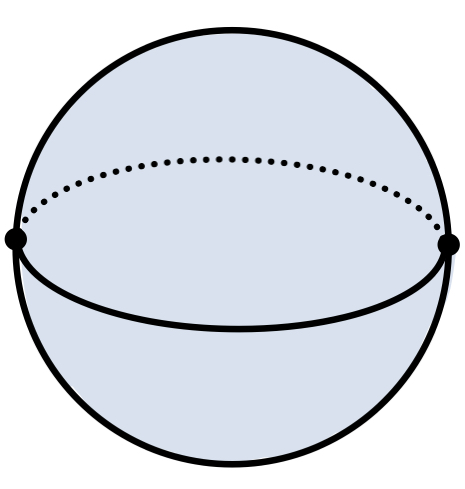}
        %  \hfill
        \hspace{10pt}
         \includegraphics[width=0.35\linewidth]{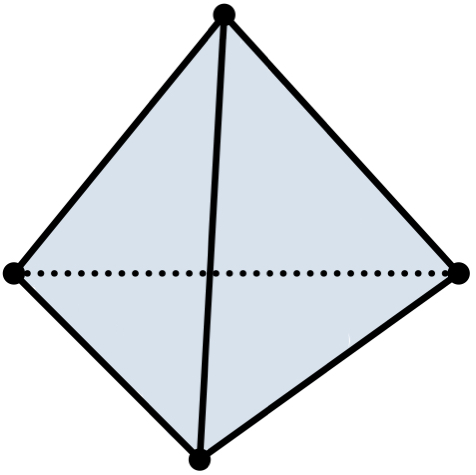}
         \caption{A sphere and an empty tetrahedron. The latter is also a simplicial complex.}
        \label{fig:1}
     \end{minipage}
     \vspace{-10pt}
\end{figure}

Consider the examples in Figure \ref{fig:1}. The shown sphere is a cell complex obtained from  two 0-cells (i.e.\ vertices), to which two 1-cells (i.e.\ edges), which form the equator, were attached. The boundary of two 2-dimensional disks (i.e.\ the two hemispheres) were glued to the equator to form a sphere. The second example is a tetrahedron with empty interior. It is a particular type of cell complex called a \emph{simplicial complex} (SC). The only 2-cells it allows are triangle-shaped. More generally, the $n$-dimensional cells of SCs are $n$-simplices, which makes them slightly more rigid structures.

% \begin{wrapfigure}{r}{0.3\textwidth}
%     \centering
%     \vspace{-28pt}
%     \begin{subfigure}[b]{0.3\textwidth}
%         \centering
%         \includegraphics[width=\textwidth]{figures/glue_disks.jpeg}     
%         \label{fig:y equals x3}
%      \end{subfigure}\\
%      %\hfill
%      \begin{subfigure}[b]{0.3\textwidth}
%         \centering
%         \includegraphics[width=\textwidth]{figures/glued_disks.jpeg}     
%         \label{fig:y equals x4}
%      \end{subfigure}\vspace{-4mm}
%     \caption{Closed 2D disks are glued to the boundary of the rings present in the graph (top). The result is a 2D regular cell complex (bottom).}\vspace{-15mm}
%     \label{fig:gluing}
%     % \vspace{-6pt}
% \end{wrapfigure}

\begin{definition}
The $\mathbf{k}$\textbf{-skeleton} of a cell complex $X$, denoted $X^{(k)}$, is the subcomplex of $X$ consisting of cells of dimension at most $k$.
\end{definition}

This definition is useful for referring for certain parts of the complex. For instance, $X^{(0)}$ contains the vertices in the complex, while $X^{(1)}$ contains the vertices \emph{and} the edges (i.e. the underlying graph). 

The combinatorial structure of the complex can be more compactly described by an incidence relation we call the \emph{boundary relation}, whose reflexive and transitive closure gives the partial order defined above. The boundary relation describes what cells are on the boundary of other cells. For instance, the edges of the sphere in Figure \ref{fig:1} are on the boundary of the 2-cells forming the two hemispheres. 

\begin{definition}
\label{def:boundary_rel}
We have the \textbf{boundary relation} $\sigma \prec \tau$ iff $\sigma < \tau$ and there is no cell $\delta$ such that $\sigma < \delta < \tau$.
\end{definition}

We can use this to define the four types of (local) adjacencies present in cell complexes. These adjacencies will be the fundamental building block of our message passing procedure. To explain these in more familiar terms, for each adjacency, we exemplify how it shows up in graphs.   

\begin{definition}[Cell complex adjacencies]
\label{def:adj}
For a cell complex $X$ and a cell $\sigma \in P_X$, we define:
\begin{enumerate}[leftmargin=*, topsep=0pt,itemsep=-0.5ex]
    \item The boundary adjacent cells $\gB(\sigma) = \{ \tau \mid \tau \prec \sigma \}$. These are the lower-dimensional cells on the boundary of $\sigma$. For instance, the boundary cells of an edge are its vertices.
    \item The co-boundary adjacent cell $\gC(\sigma) = \{ \tau \mid \sigma \prec \tau \}$. These are the higher-dimensional cells with $\sigma$ on their boundary. For instance, the co-boundary cells of a vertex are the edges it is part of.
    \item The lower adjacent cells $\ndown(\sigma) = \{ \tau \mid \exists \delta$ such that $\delta \prec \sigma$ and $\delta \prec \tau \}$. These are the cells of the same dimension as $\sigma$ that share a lower dimensional cell on their boundary. The line graph adjacencies between the edges are a classic example of this. 
    \item The upper adjacent cells $\nup(\sigma) = \{ \tau \mid \exists \delta$ such that $\sigma \prec \delta$ and $\tau \prec \delta \}$. These are the cells of the same dimension as $\sigma$ that are on the boundary of the same higher-dimensional cell as $\sigma$. The typical graph adjacencies between vertices are the canonical example here. 
\end{enumerate}
\end{definition}

\section{Cellular Weisfeiler Lehman}

\paragraph{Overview} The results in this section show how one can transform graphs into higher-dimensional cell complexes in such a way that performing colour refinement on the resulting cell complexes makes it easier to test their isomorphism. The message passing model from Section \ref{sec:CWN_MMP} will take advantage of these theoretical results. All proofs can be found in Appendix \ref{app:proofs_cellular}. 

\begin{definition}
Let $c$ be a colouring of the cells in a complex $X$ with $c_\sigma$ denoting the colour assigned to cell $\sigma \in P_X$. Define $\gB(\sigma, \tau) := \gB(\sigma) \cap \gB(\tau)$ and $\gC(\sigma, \tau) := \gC(\sigma) \cap \gC(\tau)$. We define the following multi-sets of colours:
\begin{enumerate}[leftmargin=*, topsep=0pt,itemsep=-0.5ex]
    \item The colours of the boundary cells of $\sigma$: $c_\gB(\sigma) = \ldblbrace c_\tau \mid \tau \in \gB(\sigma) \rdblbrace$. 
    \item The colours of the co-boundary cells of $\sigma$: $c_\gC(\sigma) = \ldblbrace c_\tau \mid \tau \in \gC(\sigma) \rdblbrace$.
    \item The lower adjacent colours of $\sigma$: $c_\da(\sigma) = \ldblbrace (c_\tau, c_\delta) \mid \tau \in \ndown(\sigma)$ and $\delta \in \gB(\sigma, \tau) \rdblbrace$.
    \item The upper adjacent colours of $\sigma$: $c_\ua(\sigma) = \ldblbrace (c_\tau, c_\delta) \mid \tau \in \nup(\sigma)$ and $\delta \in \gC(\sigma, \tau) \rdblbrace$.
\end{enumerate}
\end{definition}

Note that unlike in graphs and simplicial complexes, the sets $\gB(\sigma, \tau)$ and $\gC(\sigma, \tau)$ can have more than one element. For instance, two (closed) 2-cells might intersect in more than one edge (e.g. the two hemispheres in Figure \ref{fig:1}), and conversely, two edges might be on the boundary of the same two 2-cells. This illustrates the more flexible combinatorial structure of cell complexes. 

\paragraph{Cellular WL (CWL)} We consider CWL, a colour refinement scheme for cell complexes that generalises the Simplicial WL~\citep{bodnar2021weisfeiler} and WL~\citep{weisfeiler1968reduction} tests. We use $c_\sigma^t$ to refer to the colour assigned by CWL to cell $\sigma$ at iteration $t$ of the algorithm. When the input is a simplicial complex, this recovers the SWL algorithm. A step of the algorithm is graphically depicted in Figure \ref{fig:CWL} for a single cell. 
\begin{enumerate}[leftmargin=*]
    \item Given a regular cell complex $X$, all the cells $\sigma$ are initialised with the same colour. 
    \item Given the colour $c_\sigma^t$ of cell $\sigma$ at iteration $t$, we compute the colour of cell $\sigma$ at the next iteration  $c_\sigma^{t+1}$ by injectively mapping the multi-sets of colours belonging to the adjacent cells of $\sigma$ using a perfect HASH function:  $c_\sigma^{t+1} = \mathrm{HASH}\big(c_\sigma^t, c_\gB^t(\sigma), c_\gC^t(\sigma), c_\da^t(\sigma), c_\ua^t(\sigma)\big). $
    \item The algorithm stops when a stable colouring is reached. Two cell complexes are considered non-isomorphic if their colour histograms are different. Otherwise, the test is inconclusive. 
\end{enumerate}

First, we state the following theorem from \citet{bodnar2021weisfeiler} involving SWL and simplicial complexes. This theorem shows that on simplicial complexes, certain adjacencies can be pruned without affecting the non-isomorphic SCs that can be distinguished. This has important computational implications. 

\begin{theorem}
SWL without coboundary and lower-adjacencies has the same expressive power in distinguishing non-isomorphic simplicial complexes as SWL with the complete set of adjacencies. 
\end{theorem}

It is not immediately clear whether an equivalent theorem would also hold for cell complexes. This is because cells, unlike simplices, can have widely different shapes and, as described above, the adjacencies between them take more complicated forms. Nevertheless, we show that a positive result can be obtained. 

\begin{theorem}
\label{thm:sparse_cwl}
CWL without coboundary and lower-adjacencies has the same expressive power in distinguishing non-isomorphic cell complexes as CWL with the complete set of adjacencies. 
\end{theorem}

We note this does not mean that the removed adjacencies are completely redundant in practice. Even if they are not needed from a (theoretical) colour refinment perspective, they might still include important inductive biases that make them suitable for certain tasks. 

\begin{figure}
    \centering
    \includegraphics[width=0.9\textwidth]{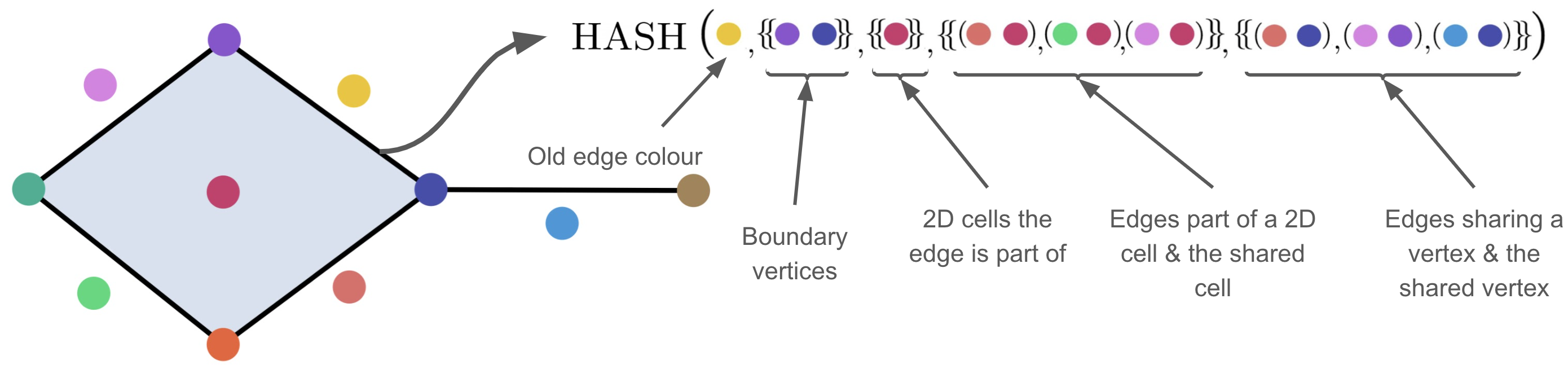}
    \caption{The CWL colouring procedure for the yellow edge of the cell complex. All cells have been assigned unique colours to aid the visualisation of the adjacencies. Note that the yellow edge aggregates long-range information from the light green edge.}
    \label{fig:CWL}
    % \vspace{-10pt}
\end{figure}

We are now interested in examining various procedures for mapping, or ``lifting'', graphs into the space of regular cell complexes. Such a procedure can be used to test the isomorphism of two graphs by performing colour refinement on the cell complexes they are mapped to. The hope is that CWL applied to these cell complexes is more powerful than WL applied to the initial graphs. We will later show that for a wide range of transformations, this is indeed the case. We start by rigorously defining what we mean by a  ``lifting''.

\begin{definition}
A \textbf{cellular lifting map} is a function $f: \gG \to \gX$ from the space of graphs $\gG$ to the space of regular cell complexes $\gX$ with the property that two graphs $G_1, G_2$ are isomorphic iff the cell complexes $f(G_1), f(G_2)$ are isomorphic.  
\end{definition}

This property ensures that testing the isomorphism of the two cell complexes is equivalent to testing the isomorphism in the input graphs. This would not be the case if two non-isomoprhic graphs were mapped to the same cell complex.

\begin{example}
\label{ex:clique_complex}
It can be verified that the function mapping each graph to its clique complex (i.e.\ every $(k+1)$-clique in the graph becomes a $k$-simplex) is a cellular lifting map. %It can be easily checked that it respects the property above. 
\end{example}

The clique complex lifting map from Example \ref{ex:clique_complex} has been used by \citet{bodnar2021weisfeiler} to show that SWL is strictly more powerful than WL. We restate this result:
\begin{theorem}
SWL with clique complex lifting is strictly more powerful than WL. 
\end{theorem}

%A natural question to ask is what other such lifting transformations make CWL strictly more powerful than WL? We now describe an entire space of lifting transformations that make CWL at least as powerful as WL.
A natural question is what other lifting transformations make CWL strictly more powerful than WL? We first describe a space of lifting transformations that make CWL at least as powerful as WL. 

\begin{definition}
A lifting map is \textbf{skeleton-preserving} if for any graph $G$, the $1$-skeleton of $f(G)$ and $G$ are isomorphic as (multi) graphs.
\end{definition}

Intuitively, skeleton-preserving liftings ensure that the additional structure added by the lifting map comes from attaching cells of dimension at least two to the graph. These mappings keep the 0-cells and 1-cells intact and are, therefore, restricted from making modifications to the input graph structure. An important remark is that for simplicial complexes, attaching simplices based on cliques present in the graph is the only possible skeleton preserving transformation. Once again, this illustrates the limitations of simplicial complexes for adding useful higher-dimensional structures to the graph. 

\begin{wrapfigure}[9]{r}{0.35\textwidth}
    \begin{subfigure}[t!]{1.0\linewidth}
        \centering
        \vspace{-.3cm}
        \includegraphics[width=1.0\textwidth]{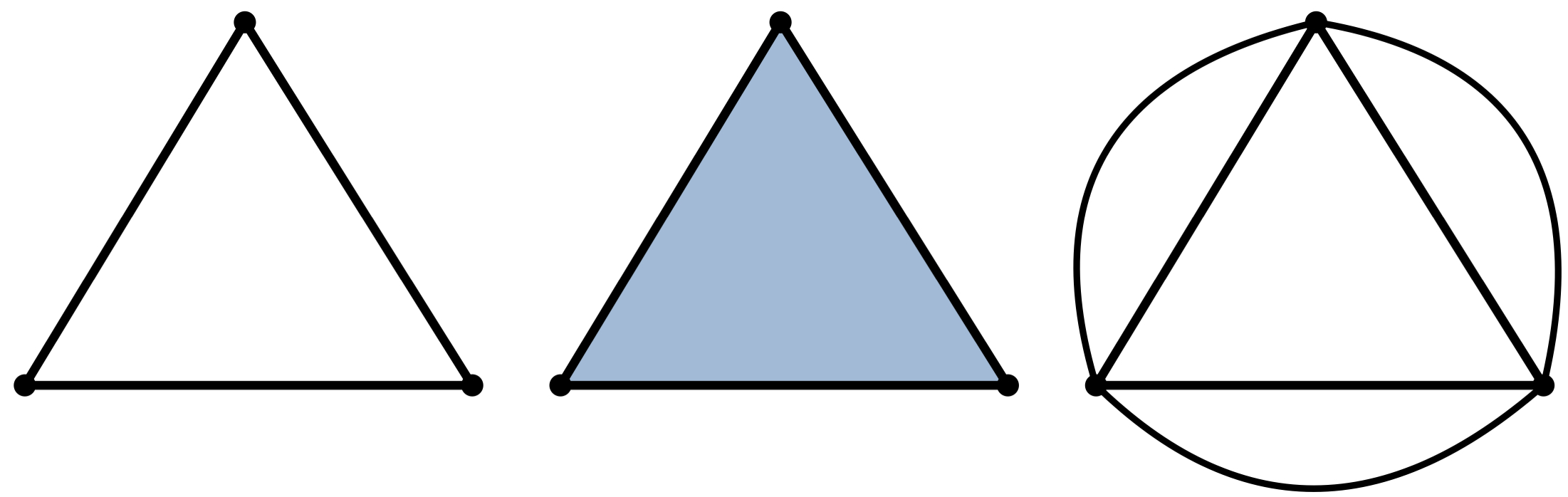}
        \vspace{-4mm}
    \end{subfigure}
     \caption{A graph, its clique complex and the graph with duplicated edges. The first map is skeleton-preserving, while the second is not.}
     \label{fig:sk_example}
\end{wrapfigure}

\begin{example}
\label{ex:skeleton-lift}
The function from Example \ref{ex:clique_complex} is also skeleton-preserving because the 1-skeleton of the clique complex of a graph is trivially isomorphic to the graph. A lifting function mapping each graph to a multi-graph where each edge is doubled by a parallel edge is not skeleton-preserving (Figure \ref{fig:sk_example}). 
\end{example}

We now show that all the maps in the skeleton-preserving class have the following desirable property: 

\begin{theorem}
\label{thm:skeleton}
Let $f$ be a skeleton-preserving lifting map. Then CWL($f$) (i.e. CWL using lifting $f$) is at least as powerful as WL in distinguishing non-isomorphic graphs.  
\end{theorem}

To prove that some of these make CWL strictly more powerful than WL, it is sufficient to find a pair of graphs that cannot be distinguished by WL, but can be distinguished by CWL. The following result gives examples of such maps. 

\begin{definition}
Let $k$-$\mathrm{CL}$, $k$-$\mathrm{IC}$, $k$-$\mathrm{C}$ be the lifting maps attaching cells to all the cliques, induced cycles and simple cycles, respectively, of size at most $k$.  
\end{definition}

\begin{corollary}
\label{cor:WL_lifting_maps}
For all $k \geq 3$, CWL($k$-$\mathrm{CL}$), CWL($k$-$\mathrm{IC}$) and CWL$(k$-$\mathrm{C})$ are strictly more powerful than WL.   
\end{corollary}

We note that this is not a complete list. For instance, the result can also be extended to combinations of the above or other transformations. We can also relate CWL to the higher-order 3-WL test.  

\begin{theorem}
\label{thm:lifting3WL}
There exists a pair of graphs indistinguishable by 3-WL but distinguishable by CWL($k$-$\mathrm{CL}$) with $k \geq 4$, CWL($k$-$\mathrm{IC}$) with $k \geq 4$ and CWL$(k$-$\mathrm{C})$ with $k \geq 8$.
\end{theorem}

Finally, we conclude this section by showing how CWL can achieve a superior expressive power compared to SWL. This result is proven by Corollary \ref{cor:cwl_better_than_swl} in the Appendix. 

\begin{theorem}
Let $k$-$\mathrm{CL}$ $\cup$ $k$-$\mathrm{IC}$ and $k$-$\mathrm{CL}$ $\cup$ $k$-$\mathrm{C}$ denote combined liftings attaching cells to the union of the specified substructures. CWL($k_1$-$\mathrm{CL} \cup k_2$-$\mathrm{IC}$) and CWL($k_1$-$\mathrm{CL} \cup k_2$-$\mathrm{C}$) are strictly more powerful than SWL($k_1$-$\mathrm{CL}$) for all $k_2 \geq 5$. 
\end{theorem}

\section{Molecular Message Passing with CW Networks}
\label{sec:CWN_MMP}

We now describe CW Networks with an applied focus on molecular graphs to ground the discussion. Therefore, from now on we assume the use of the skeleton-preserving lifting transformation that attaches 2-cells to all the induced cycles (i.e. chordless cycles) in the graph as in Figure \ref{fig:gluing}. This leads to a message passing procedure involving atoms (vertices / 0-cells), the bonds between atoms (edges / 1-cells) and chemical rings (induced cycles / 2-cells). Additionally, in virtue of Theorem \ref{thm:sparse_cwl}, we consider only the boundary and upper adjacencies between these cells without sacrificing the expressive power. The equations for the other adjacencies, which we do not use, can be found in Appendix \ref{app:proofs_cellular}. We note however, that the theoretical results in this section are general and not particular to these specific choices of adjacencies and lifting transformation. 

\paragraph{Molecular Message Passing} The cells in our CW Network receive two types of messages: 
\begin{align}
    &m_{\gB}^{t+1}(\sigma) = \text{AGG}_{\tau \in \gB(\sigma)}\Big(M_{\gB}\big(h_{\sigma}^{t}, h_{\tau}^{t}\big)\Big) \quad
    &m_{\ua}^{t+1}(\sigma) = \text{AGG}_{\tau \in \nup(\sigma), \delta \in \gC(\sigma, \tau)}\Big(M_{\ua}\big(h_{\sigma}^{t}, h_{\tau}^{t}, h_{\delta}^t\big)\Big). \nonumber
\end{align}
The first specifies messages from atoms to bonds and from bonds to rings. The second type of message, specifies messages between atoms connected by a bond and messages between bonds that are part of the same ring (Figure \ref{fig:message_passing}). Note that for the second type of adjacency, when two atoms communicate, we include the features of the bond between them. Similarly, when two bonds communicate, we include the features of the ring they communicate through. 
The update operation takes into account these two types of incoming messages and updates the features of the cells:
\begin{equation}\label{eq:cwn_update}
    h_{\sigma}^{t+1} = U\Big(h_{\sigma}^{t}, m_{\gB}^{t}(\sigma), m_{\uparrow}^{t+1}(\sigma) \Big) . 
\end{equation}
To obtain a global embedding for a cell complex $X$ from a model with $L$ layers, the readout function takes as input the separate multi-sets of features corresponding to the atoms, bonds and the rings:
\begin{equation}
    h_X = \text{READOUT}(\ldblbrace h_{\sigma}^L \rdblbrace_{dim(\sigma)=0}, \ldblbrace h_{\sigma}^L \rdblbrace_{dim(\sigma)=1}, \ldblbrace h_{\sigma}^L \rdblbrace_{dim(\sigma)=2}).
    % h_X = \text{READOUT}(\ldblbrace h_{\sigma_0}^L \rdblbrace, \ldblbrace h_{\sigma_1}^L \rdblbrace, \ldblbrace h_{\sigma_2}^L \rdblbrace)
\end{equation}

\begin{wrapfigure}[]{r}{0.22\textwidth}
    \begin{subfigure}[t!]{1.0\linewidth}
        \centering
        \vspace{-10pt}
        \includegraphics[width=1.0\textwidth]{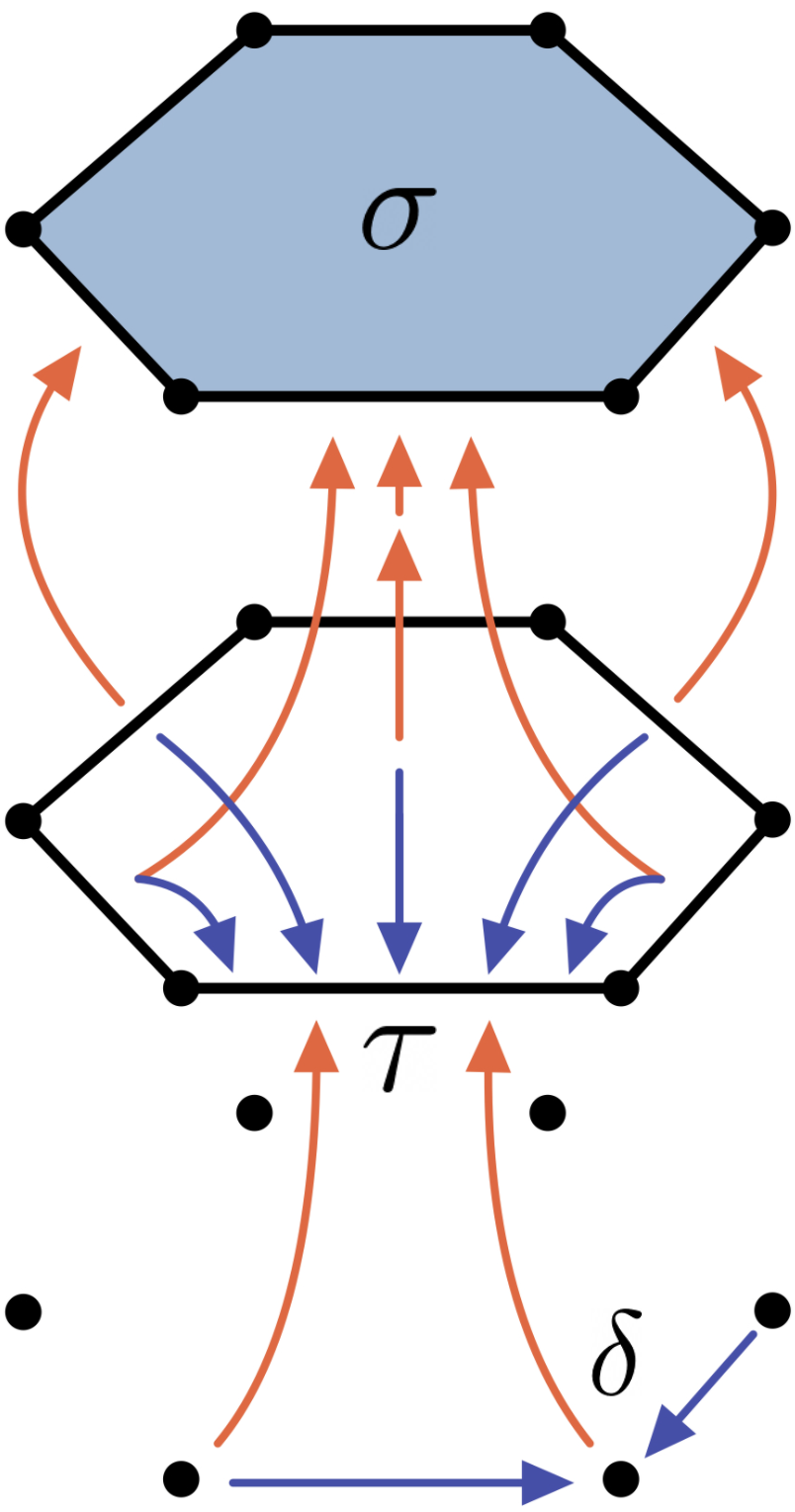}
    \end{subfigure}
    \caption{Hierarchical depiction of the message passing procedure. \textcolor{orange}{Orange} arrows indicate boundary messages received by cells $\sigma$ and $\tau$, while \textcolor{BlueViolet}{blue} ones show upper messages received by cells $\tau$ and $\delta$.}
    \label{fig:message_passing}
    \vspace{-40pt}
\end{wrapfigure}

\paragraph{Expressivity} Naturally, the ability of CWNs to distinguish non-isomorphic regular cell complexes is bounded by CWL. Similarly to GNNs and WL, CWNs can also be shown to be as powerful as CWL as long as they are equipped with a sufficient number of layers and the parametric local aggregators they use can learn to be injective. Multiple such multi-set aggregators \citep{GIN, Corso2020_PNA} are known to exist and can be directly employed in our model.   

\begin{theorem}
\label{thm:CWandCWN}
CW Networks are at most as powerful as CWL. Additionally, when using injective neighbourhood aggregators and a sufficient number of layers, CWNs are as powerful as CWL.
\end{theorem}

Corollary~\ref{cor:WL_lifting_maps} states that CWL is strictly more powerful than the standard WL when the lifting procedure attaches 2-cells to induced cycles of size $k \geq 3$. As a consequence of Theorem~\ref{thm:CWandCWN}, this result also holds for molecular message passing CWNs equipped with injective aggregators. In practice, $k$ is to be considered as a standard hyperparameter, and its choice can either be driven by validation set performance, or by domain knowledge (if available).

\paragraph{Symmetries} Given a graph $G$ with adjacency matrix $\mA$ and feature matrix $\mX$, a function $f$ is (node) permutation equivariant if $\mP f(\mA, \mX) = f(\mP \mA \mP^T, \mP \mX)$, for any permutation matrix $\mP$. GNN layers respect this equation, which ensures 
they compute the same functions up to a permutation (i.e. relabeling) of the nodes. Similarly, it can be shown that CW Networks are equivariant with respect to permutations of the cells and corresponding permutations of the boundary relations $\sigma \prec \tau$ between cells. We define this notion of equivariance more formally in Appendix \ref{app:symmetries}.    

\begin{theorem}
\label{thm:CWequivariant}
CW Network layers are cell permutation equivariant. 
\end{theorem}

\paragraph{Long-Range Interactions} 

Several graph-related tasks require the ability to capture long-range interactions between nodes. For instance, certain molecular properties depend on atoms placed on the opposite sides of a ring \citep{gilmer2017neural, ramakrishnan2014quantum}. 
% The notion of ``problem radius'', informally captures this aspect and dictates a lower bound on the number of stacked message passing layers that is required by traditional GNNs to solve the task. 
As a consequence of the coupling between the input and computational graphs, $L$ message passing operations are necessary in GNNs to let a node receive information from an $L$-hops distant node. 
% This intrinsic limitation has encouraged the design of \emph{deep} GNNs, which, however, incur several drawbacks, including oversmoothing~\citep{li2018deeper} and oversquashing~\citep{alon2020bottleneck} of node signals. 
In contrast, our hierarchical message passing scheme requires \emph{at most} $L$ layers since $2$-cells create shortcuts. 
%In principle, one single message passing step would suffice to transfer information between edges that may be significantly distant in the original graph. 
% In view of this observation, the number of required cellular layers depends on the choice of lifting transformations, other than the problem radius. 
For example, a constant number of CWN layers ($3$) is enough to capture dependencies between atoms on the opposite sides of a ring, independently of the ring size. In Section~\ref{ss:synth} we verify this in a controlled scenario. Additional experiments on real world graphs in Section~\ref{ss:real} confirm that it can achieve state-of-the-art performance with a limited number of layers.

\paragraph{Anisotropic Filters} 
%Many common GNNs \citep{gilmer2017neural, kipf2017graph} use symmetric  convolutional filters. Therefore, all node neighbours are treated equally, irrespective of their location in the graph. In contrast, CWNs naturally integrate the information from the higher-order cells, which can inform how message passing between nodes is performed. For instance, bond features can learn to encode their membership to a ring and communicate directly with other bonds present in the ring. Consequently, the messages between atoms connected through these bonds are modulated by the presence of the ring as well as by the presence of other nodes and bonds part of that ring. 
Due to the lack of a canonical ordering between neighbours, many common GNNs use symmetric convolutional kernels, resulting in isotropic filters treating neighbours equally. Recent works have proposed to address this limitation by employing additional structural information~\citep{beaini2020directional, bouritsas2020improving}. CWNs also implicitly achieve this form of anisotropy by integrating information from the higher-order cells and their associated substructures into the message passing procedure. For instance, bond features can learn to encode their membership to a ring and also communicate directly with other bonds present in the ring. Consequently, the messages between atoms connected through these bonds are modulated by the presence of the ring as well as by the presence of other nodes and bonds part of that ring.

\paragraph{CWNs as Generalised Convolutions}
Our message passing scheme can be seen a (non-linear) generalisation of linear diffusion operators on cell complexes.
% which, in the GNN literature, are broadly addressed as `convolutional'. Representative graph models in this class include ChebNet~\citep{defferrard2016convolutional} and the popular GCN~\citep{kipf2017graph}. 
% These models diffuse information locally by applying a normalised linear operator built as a function of a Laplacian operator. 
% Inspired by these seminal approaches, 
Recent works~\citep{ebli2020simplicial, bunch2020simplicial} have introduced convolutional operators on SCs by employing the Hodge Laplacian \citep{schaub2020random},
a generalisation of the graph Laplacian. 
% which can be defined in terms of the 
% boundary operators describing the 
% boundary relations in the complex.
By leveraging on the cellular Sheaf Laplacian \citep{HGh19}, a similar construction can be extended to cell complexes to define cellular convolutional operators. In Appendix~\ref{app:convs} we discuss this approach and show that our cellular message passing scheme subsumes it. This represents a promising avenue for studying CWNs from a spectral perspective, an endeavour we leave for future work. 

\paragraph{Computational Complexity} When considering cells of a constant maximum dimension and boundary size, the computational complexity of the message passing scheme is linear in the size of the input complex. For the molecular applications we are interested in, the average number of rings per molecule is upper bounded by a small constant (e.g. three for MOLHIV), so the size of the complex is approximately the same as the size of the graph. Therefore, in this setting, the computational complexity of the model is similar to that of message passing GNNs. Separately of this, the one-time preprocessing step of computing the lifting of the graphs should also be considered. The $C$ induced cycles in a graph can be listed in $\gO\big((|E| + |V|C) \text{ polylog } |V|\big)$ time \citep{ferreira2014amortized}. Again, given that $C$ is upper bounded by a small constant for the molecular datasets of interest in this work, the complexity of the lifting procedure is also almost linear in the size of the graph. A more detailed analysis backed up by wall-clock time experiments is given in Appendix \ref{app:complexity}.

% These can be listed using, for instance, the junction tree representation~\citep{jin2018junction,Fey2020_himp}, obtained by specialised libraries such as RDKit~\citep{Landrum2016RDKit2016_09_4}. In our experiments on molecular benchmarks the lifting procedure has been computationally negligible w.r.t. training runtimes. We include a detailed computational analysis and wall-clock time data both for the message passing and the lifting transformation in Appendix \ref{app:proofs_cellular}. 

\section{Experiments}
\label{sec:results}

In this section we validate the theoretical and empirical properties of our proposed message passing scheme in controlled scenarios as well as in real-world graph classification problems, with a focus on large scale molecular benchmarks. For simplicity, in all experiments we employ a model which stacks CWN layers with local aggregators as in GIN~\citep{GIN}. We name our architecture ``Cell Isomorphism Network'' (CIN). $0$-cells are always endowed with the original node features; higher-dimensional cells are populated in a benchmark specific manner. See Appendix~\ref{app:results} for details on feature initialisation, message passing and readout operations, hyperparameters, implementation and benchmark statistics.

\begin{figure}[t]
    \centering
    \begin{subfigure}[t!]{0.44\textwidth}
        \centering
        \includegraphics[width=\textwidth]{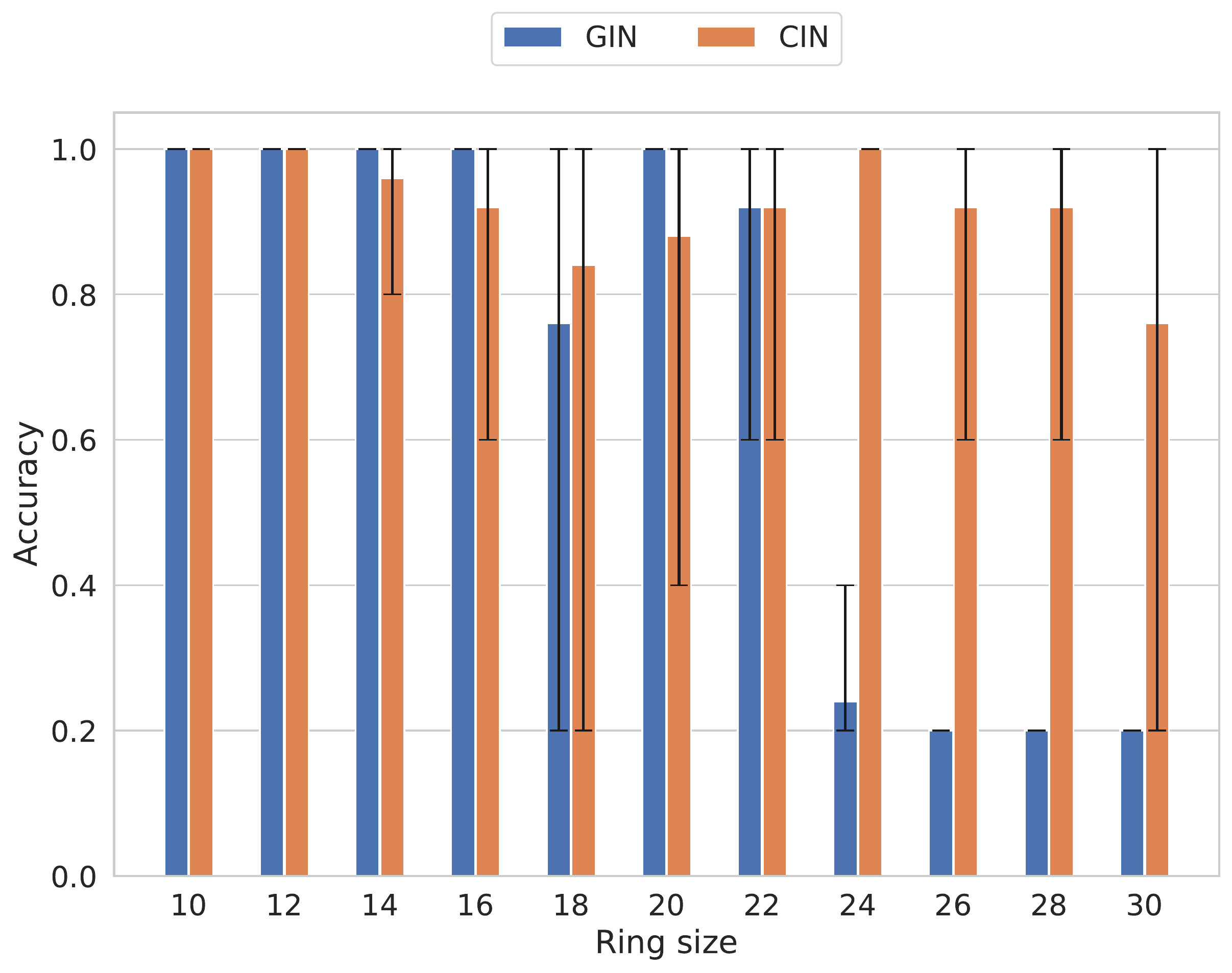}
        \caption{RingTransfer Results. Accuracy is over $5$ balanced classes. A score of $0.2$ is equivalent to a random guess. Error bars show the min and max. Our model obtains high-scores in average even for large rings despite using only three layers.}
        \label{fig:rt}
    \end{subfigure}
    \hfill
    \begin{subfigure}[t!]{0.53\textwidth}
        \centering
        \includegraphics[width=\textwidth]{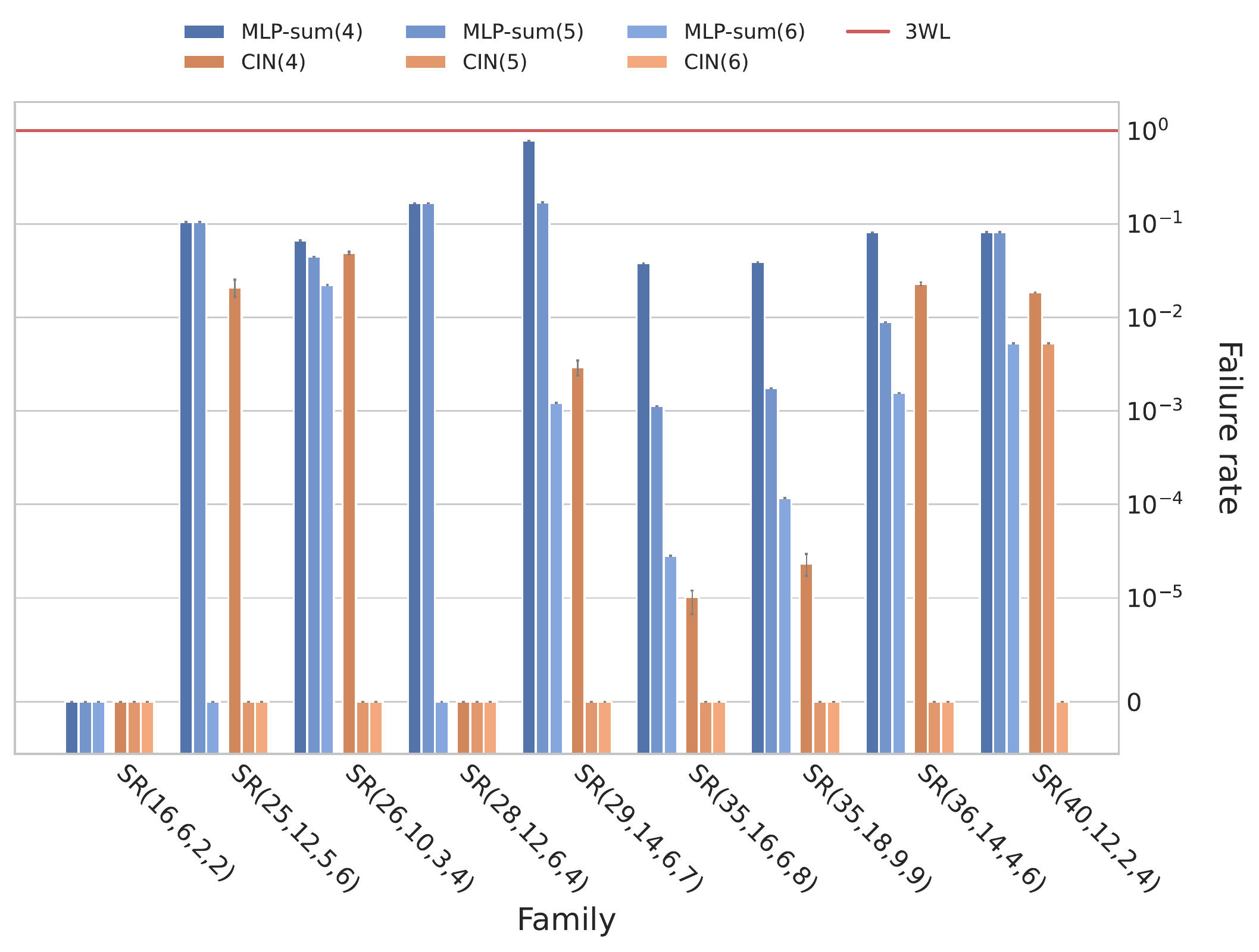}
        \caption{Failure rates on the SR isomorphism task, \emph{the smaller the better} (mean and std-error over $5$ runs). In parantheses, for each model, the maximum size $k$ of rings lifted to $2$-cells.}
        \label{fig:sr_iso}
    \end{subfigure}
    \caption{Results on the RingTransfer and SR synthetic benchmarks.}
    \label{fig:synth}
    \vspace{-12pt}
\end{figure}

\subsection{Synthetic Benchmarks}\label{ss:synth}

% In order to empirically validate the expressive power of our architecture, we focus on the task of graph isomorphism over the two classes of (regular) Circular Skip Link (CSL) and Strongly Regular (SR) graphs.

\paragraph{CSL} Circular Skip Link dataset was first introduced in~\cite{pmlr-v97-murphy19a} %\citet{murphy2019relational} 
and has been recently adopted as a reference benchmark to test the expressivity of GNNs~\citep{dwivedi2020benchmarkgnns}. It consists of 150 $4$-regular graphs from 10 different isomorphism classes, which we need to predict. 
% Circular Skip Link graphs of isomorphism class $\mathcal{G}_{N,C}$, where $N=48$ (the number of nodes) and $C \in \{2, 3, 4, 5, 6, 9, 11, 12, 13, 16\}$ (the skip number). 
Unsolvable by the WL test and message passing approaches \citep{pmlr-v97-murphy19a, Chen2019_ring_gnn}, we use it to validate the expressive power of CWNs. 

\begin{wraptable}[7]{l}{0.425\textwidth}
        \centering
        \vspace{-13pt}
        \begin{minipage}[t]{1.0\linewidth}
        % \vspace{-.4cm}
         \caption{Classification accuracy on CSL.}\label{tab:csl}
          \resizebox{\columnwidth}{!}{
          \begin{tabular}{l  ccc}
            \toprule
            Method & 
            Mean &
            Min &
            Max \\
            \midrule
            
            MP-GNNs & 
            10.000$\pm$0.000 &
            10.000 &
            10.000 \\
            
            RingGNN & 
            10.000$\pm$0.000 &
            10.000 &
            10.000 \\
            
            3WLGNN &
            97.800$\pm$10.916 &
            30.000 &
            100.000 \\

            \midrule
            
            CIN (Ours) & 
            100.000$\pm$0.000 &
            100.000 &
            100.000 \\
            \bottomrule
          \end{tabular}%
          }
    \end{minipage}
% \end{table}
\end{wraptable}

We follow the same evaluation setting as ~\citet{dwivedi2020benchmarkgnns}: $5$-fold cross validation procedure and $20$ different random weight initialisations. For our model, we set the maximum ring size $k=8$. In Table~\ref{tab:csl} we follow the common practice on this dataset and report the mean, minimum and maximum test accuracy obtained by CIN over the $100$ runs, along with the results by the baselines presented in~\citet{dwivedi2020benchmarkgnns}.
MP-GNNs, 
that is classic message passing GNNs (GAT~\citep{velivckovic2017graph}, MoNet~\citep{Monti_etal2017}, GIN~\citep{GIN}, etc.), and RingGNN~\citep{Chen2019_ring_gnn} perform as random guessers. In contrast, our model is able to identify the isomorphism class of each test graph in every run while featuring only a fraction of the computational complexity of 3WLGNN, the best performing reference baseline~\citep{dwivedi2020benchmarkgnns, maron2019provably}.

% 3WLGNN~\citep{maron2019provably}. MP-GNNs perform as random guessers over the $10$ distinct isomorphism classes, due to their discriminative power upper-bounded by the $1$-WL test. The same result is obtained by the RingGNN. Without Laplacian positional encodings\citep{dwivedi2020benchmarkgnns} neither of these methods are unable to distinguish the classes.

\paragraph{SR} Similarly to~\citet{bodnar2021weisfeiler} and~\cite{bouritsas2020improving}, we consider Strongly Regular graphs within the same family as hard examples of non-isomorphic graphs we seek to distinguish. Any pair of graphs within the same family cannot provably be distinguished by 3-WL test~\citep{bodnar2021weisfeiler, bouritsas2020improving}. We reproduce the same experimental setting of~\citet{bodnar2021weisfeiler}. In particular, we consider $9$ distinct SR families\footnote{Data available at: \url{http://users.cecs.anu.edu.au/~bdm/data/graphs.html}.} and run our model untrained on the cell complex lifting of each graph, with $k = 4, 5, 6$. $0$-cells (nodes) are initialised with a constant unitary signal, while $1$- and $2$-cells are initialised with the sum of the contained $0$-cells. We additionally run an MLP baseline with sum readout to appreciate the contribution of message passing. We report the percentage of non-distinguished pairs in Figure~\ref{fig:sr_iso}. 
%, along with standard error, for $5$ different random weight initialisations. Even untrained, 
Contrary to 3-WL, both CIN and the MLP baseline are able to distinguish many pairs across all families, with better performance attained for larger $k$. For $k = 6$, we observed CIN to disambiguate all pairs in all families ($0.0$\% failure rate). Despite the strong results achieved by the baseline, we found CIN to always distinguish a larger number of non-isomorphic pairs for the same values of $k$, this confirming the importance of cellular message passing.
%CIN is able to distinguish all graphs in all families for $k = 6$ ($0.0$\% failure rate). Although the baseline model achieves relatively low failure rates for the same values of $k$, we observe that our message passing scheme distinguishes a larger number of non-isomorphic pairs in all settings.

\paragraph{RingTransfer} In order to empirically validate the ability of CIN to capture long-range node dependencies, we additionally design a third synthetic benchmark dubbed as `RingTransfer'. Graphs in this dataset are chordless cycles (rings) of size $k$. In each graph we mark two special nodes as \textbf{target} and \textbf{source}, always placed at distance $\lfloor \frac{k}{2} \rfloor$. The task is for \textbf{target} to output the one-hot encoded label assigned to \textbf{source}. All other nodes in the ring are assigned a unitary constant feature vector. A model has to learn to transfer the information contained in \textbf{source} to the opposite side of the ring, where \textbf{target} resides. We initialise $1$- and $2$-dimensional cells with a null signal. In Figure~\ref{fig:rt} we show the performance of a $3$-layer CIN as a function of the ring size $k$, along with that of GIN~\citep{GIN} baselines equipped with $\lfloor \frac{k}{2} \rfloor$ stacked layers. We observe that our model learns to solve the task with only $3$ computational steps, independent of $k$. As for GIN, we observed degradation in the performance for $k \geq 24$, up to complete failure. We hypothesise this to be due to the difficulties of training such a deep GNN ($\geq 12$ layers). We further verify the (theoretically expected) failure of GIN (not included) when endowed with less than $\lfloor \frac{k}{2} \rfloor$ layers. 

\subsection{Real-World Graph Benchmarks}\label{ss:real}

\begin{table}[!t]
    \centering
    % \captionsetup[table]{skip=\belowcaptionskip}
    \caption{TUDatasets. The first section of the table includes the accuracy of graph kernel methods, while the second includes GNNs. The top three are highlighted by \textbf{\textcolor{red}{First}}, \textbf{\textcolor{violet}{Second}}, \textbf{Third}.}
    % \hspace{-2cm}
    \label{tab:tud_datasets}
    \resizebox{\linewidth}{!}{%
    \begin{tabular}{l | lllll | lll}
    % \begin{adjustwidth}{-.5in}{-.in}
        \toprule
        Dataset & 
        MUTAG &
        PTC &
        PROTEINS &
        NCI1 &
        NCI109 &
        IMDB-B & 
        IMDB-M & 
        RDT-B \\
        \midrule       
        %\parbox[t]{1mm}{\multirow{10}{*}{\rotatebox[origin=c]{90}{same splits}}} 
        RWK \citep{gartner2003graph} & 
         79.2$\pm$2.1 & 
         55.9$\pm$0.3 & 
         59.6$\pm$0.1 & 
         $>$3 days & 
         N/A & 
         N/A &
         N/A &
         N/A \\
        %  N/A\\
        
        GK ($k=3$) \citep{shervashidze2009efficient} &
        81.4$\pm$1.7 & 
        55.7$\pm$0.5 &
        71.4$\pm$0.3 & 
        62.5$\pm$0.3 & 
        62.4$\pm$0.3 &
        N/A &
        N/A &
        N/A \\

        PK \citep{neumann2016propagation} & 
         76.0$\pm$2.7& 
         59.5$\pm$2.4 &
         73.7$\pm$0.7 & 
         82.5$\pm$0.5 & 
         N/A & 
         N/A & 
         N/A &
         N/A \\

        WL kernel \citep{shervashidze2011weisfeiler} &
          90.4$\pm$5.7 & 
          59.9$\pm$4.3 & 
          75.0$\pm$3.1 & 
          \first{86.0}$\pm$1.8 & 
          N/A &
          73.8$\pm$3.9 &
          50.9$\pm$3.8 &
          81.0$\pm$3.1 \\

        \midrule
         
        DCNN \citep{DCNN_2016} & 
        N/A&  
        N/A &
        61.3$\pm$1.6 &
        56.6$\pm$1.0 &
        N/A &
        49.1$\pm$1.4 &
        33.5$\pm$1.4 &
        N/A \\

        DGCNN \citep{zhang2018end} & 
        85.8$\pm$1.8 & 
        58.6$\pm$2.5 & 
        75.5$\pm$0.9 & 
        74.4$\pm$0.5 & 
        N/A &
        70.0$\pm$0.9 & 
        47.8$\pm$0.9 &
        N/A \\
        
        IGN \citep{maron2018invariant} &
        83.9$\pm$13.0 &
        58.5$\pm$6.9 &
        \third{76.6}$\pm$5.5 &
        74.3$\pm$2.7 & 
        \third{72.8}$\pm$1.5 & 
        72.0$\pm$5.5 & 
        48.7$\pm$3.4 &
        N/A \\
        % N/A\\
        
        GIN \citep{GIN} & 
        89.4$\pm$5.6 & 
        64.6$\pm$7.0	& 
        76.2$\pm$2.8 &
        82.7$\pm$1.7 &
        N/A & 
        75.1$\pm$5.1 &
        52.3$\pm$2.8 &
        \first{92.4}$\pm$2.5 \\

        PPGNs \citep{maron2019provably} &
        \third{90.6}$\pm$8.7 &
        66.2$\pm$6.6 &
        \first{77.2}$\pm$4.7 & 
        83.2$\pm$1.1 & 
        \second{82.2}$\pm$1.4 &
        73.0$\pm$5.8 & 
        50.5$\pm$3.6 &
        N/A \\

        Natural GN \citep{de2020natural} &
        89.4$\pm$1.6 &
        \third{66.8}$\pm$1.7 &
        71.7$\pm$1.0 &
        82.4$\pm$1.3 &
        N/A &
        73.5$\pm$2.0 &
        51.3$\pm$1.5 &
        N/A\\

        GSN \citep{bouritsas2020improving} &
        \second{92.2} $\pm$ 7.5 &
        \first{68.2} $\pm$ 7.2 & 
        \third{76.6} $\pm$ 5.0 & 
        \third{83.5} $\pm$ 2.0 &
        N/A & % <-- NCI109
        \first{77.8} $\pm$ 3.3 & 
        \first{54.3} $\pm$ 3.3 &
        N/A \\
        
        SIN \citep{bodnar2021weisfeiler} & 
        N/A  &
        N/A & 
        76.4 $\pm$ 3.3 & 
        82.7 $\pm$ 2.1 &
        N/A &
        \second{75.6} $\pm$ 3.2 & 
        \third{52.4} $\pm$ 2.9 &
        \third{92.2} $\pm$ 1.0 \\

        \midrule
        
        {{\bf CIN} (Ours)} & 
        \first{92.7} $\pm$ 6.1 &
        \first{68.2} $\pm$ 5.6 &
        \second{77.0} $\pm$ 4.3 &
        \second{83.6} $\pm$ 1.4 &
        \first{84.0} $\pm$ 1.6 &
        \second{75.6} $\pm$ 3.7 & 
        \second{52.7} $\pm$ 3.1 &
        \first{92.4} $\pm$ 2.1 \\
        
        \bottomrule

    %\end{adjustwidth}
    \end{tabular}
    }
\end{table}

\paragraph{TUD} We test our model on $8$ TUDataset benchmarks \citep{morris2020tudataset} with small and medium sizes from biology (\textbf{PROTEINS}  \citep{dobson2003distinguishing,borgwardt2005protein}), chemistry (i.e. molecules -- \textbf{MUTAG} \citep{kazius2005derivation,riesen2008iam}, \textbf{PTC}, \textbf{NCI1} and \textbf{NCI109} \citep{wale2008comparison}) to social networks (\textbf{IMDB-B}, \textbf{IMDB-M}, \textbf{RDT-B}). 
We consider induced cycle of size up to $k=6$ for our graph lifting procedure. We initialise node (and $0$-cell) features as described in~\citet{GIN}, and higher dimensional cells by averaging or summing the features of the included $0$-cells. The training setting and evaluation procedure follow those in \citet{GIN}. 
% We compare our CIN to the baseline methods:
% RWK \citep{gartner2003graph},
% GK (with $k=3$)  \citep{shervashidze2009efficient},
% PK \citep{neumann2016propagation},
% WL kernel \citep{shervashidze2011weisfeiler}, DCNN \citep{DCNN_2016}, DGCNN \citep{zhang2018end}, IGN \citep{maron2018invariant}, GIN \citep{GIN}, PPGNs \citep{maron2019provably}, Natural GN \citep{de2020natural}, GSN \citep{bouritsas2020improving}, and SIN \citep{bodnar2021weisfeiler}. 
We report the results in Table~\ref{tab:tud_datasets}. CIN compares more than favourably with the baselines, displaying strong empirical performance on all benchmarks. The mean accuracy of CIN ranks top on four out of eight datasets. On the remaining datasets, CIN achieves the second place. We observe that the best results are on datasets from the biological and chemical domains, where rings play a relevant role.
% FF: I've found MUTAG to be 92.8 in fact. maybe we could update in future versions

\paragraph{ZINC} 
We study the effectiveness of cellular message passing on larger scale molecular benchmarks from the ZINC database \citep{ZINCdataset}. \textbf{ZINC} (12k graphs) and \textbf{ZINC-FULL} (250k graphs) \citep{dwivedi2020benchmarkgnns,jin2018junction,you2018graph,gomez2018automatic} are two graph regression task datasets for drug constrained solubility prediction. In these experiments, we consider rings up to size $k=18$. We follow the training and evaluation procedures in~\citep{dwivedi2020benchmarkgnns}. Our experiments encompass different scenarios, examine the impact of ablating edge features and of constraining the parameter budget of the architecture to $100$k. All results are illustrated in Table~\ref{tab:mol_dataset} where we also include the results for \textbf{ZINC-FULL} obtained by the same exact architectures. Our model exhibits particularly strong performance on these benchmarks: it attains state-of-the-art results on both the two dataset variants, outperforming other models by a significant margin. CIN attains strong results even when constrained by the parameter budget. It still achieves state-of-the-art performance on \textbf{ZINC} and is on-par with the best unconstrained baseline under edge-feature ablation. 
% These results suggest how the inductive bias intrinsic in the employed cellular message passing scheme are particularly suited to efficiently model molecular graphs.

\paragraph{Mol-HIV} We additionally test our model on the molecular \textbf{ogbg-molhiv} dataset from the Open Graph Benchmark~\citep{hu2020open} ($41$k graphs). The task is to predict the capacity of compounds to inhibit HIV replication. Rings of size up to $k=6$ are considered as $2$-cells. We take the architecture in~\cite{Fey2020_himp} as reference and replicate the same hyperparameter setting in our model, including the use of only $2$ message passing layers. We report the mean of test ROC-AUC metrics at the epoch of best validation performance for $10$ random weight initialisations. Similarly to ZINC, we experiment with a ``small'' model whose number of parameters is constrained in the order of $100$k. Table~\ref{tab:mol_dataset} displays the results. CIN significantly outperforms other strong GNN baselines, even when constrained by the parameter budget. Consistently with~\cite{Fey2020_himp}, we observe that only two layers are sufficient when performing hierarchical message passing across meso-scale structures such as rings.

\begin{table}[!t]
    \centering
    \begin{minipage}[t]{0.80\textwidth}
        \centering
         \caption{ZINC (MAE), ZINC-FULL (MAE) and Mol-HIV (ROC-AUC).}
        \label{tab:mol_dataset}
           \resizebox{\columnwidth}{!}{
          \begin{tabular}{l cccc}
            \toprule
            \multirow{2}{*}{Method} & 
            
            \multicolumn{2}{c}{ZINC $\downarrow$} &
            ZINC-FULL $\downarrow$& 
            MOLHIV $\uparrow$\\
            
            &
            No Edge Feat. &
            With Edge Feat. &
            All methods &
            All methods \\
            \midrule
            
            GCN \citep{kipf2017graph} & 
            0.469$\pm$0.002 &
            N/A &
            N/A &
            76.06$\pm$0.97 \\

            GAT \citep{velivckovic2017graph} & 
            0.463$\pm$0.002 &
            N/A &
            N/A &
            N/A \\
            
            GatedGCN \citep{bresson2017residual} &
            0.422$\pm$0.006 & 
            0.363$\pm$0.009 &
            N/A &
            N/A \\

            GIN \citep{GIN}  & 
            0.408$\pm$0.008 &
            0.252$\pm$0.014 &
            0.088$\pm$0.002 &
            77.07$\pm$1.49 \\
            
            PNA \citep{Corso2020_PNA} & 
            0.320$\pm$0.032 & 
            0.188$\pm$0.004 &
            N/A &
            79.05$\pm$1.32 \\
            
            DGN \citep{beaini2020directional} & 
            0.219$\pm$0.010 &
            0.168$\pm$0.003 &
            N/A &
            79.70$\pm$0.97 \\
            
            HIMP \citep{Fey2020_himp} &
            N/A &
            0.151$\pm$0.006 &
            0.036$\pm$0.002 &
            78.80$\pm$0.82\\

            GSN \citep{bouritsas2020improving} & 
            0.139$\pm$0.007 &
            0.108$\pm$0.018 &
            N/A &
            77.99$\pm$1.00 \\
            
            \midrule
            
            \textbf{CIN-small} (Ours) & 
            0.139$\pm$0.008 &
            0.094$\pm$0.004 &
            0.044$\pm$0.003 &
            80.55$\pm$1.04 \\
            
            \textbf{CIN} (Ours) & 
            \textbf{0.115}$\pm$\textbf{0.003} &
            \textbf{0.079}$\pm$\textbf{0.006} &
            \textbf{0.022}$\pm$\textbf{0.002} &
            \textbf{80.94}$\pm$\textbf{0.57} \\
            \bottomrule
          \end{tabular}%
          }
    \end{minipage}
    \vspace{-11pt}
\end{table}

% \vspace{-5pt}
\section{Related Work, Discussion and Conclusion}
\label{sec:conclusion}
% \vspace{-5pt}

\paragraph{Cell complex models} Recent works have proposed the generalisation of GNNs to simplicial complexes \citep{ebli2020simplicial, bunch2020simplicial, glaze2021principled, hajij2021simplicial}. All these %convolutional 
simplicial methods are subsumed by the model in~\citet{bodnar2021weisfeiler}, which CWNs in turn subsume. 
% The authors also propose a simplicial colouring procedure to characterise its expressive power. Graph applications of this approach are limited by the rigid combinatorial structure of SCs. %, which does not allow more flexible liftings. 
To the best of our knowledge, \citet{hajij2020cell} is the only other example of message passing on cell complexes, but this work does not study the expressive power of the proposed scheme, neither it experimentally validates its performance. In contrast, our work comprehensively characterises the expressiveness of cellular message passing, and introduces a theoretically grounded and empirically effective framework to apply it on graph structured data in a way to address several limitations of standard Graph Neural Networks.
%
% In contrast to our model, their work is %however 
% a proof of concept and does not include any experimental or theoretical analysis.
%
% extends the spectral construction in~\cite{defferrard2016convolutional} and design convolutions from the simplicial Hodge Laplacian. A normalised form of such operator has been used in~\cite{bunch2020simplicial} to extend the GCN architecture~\citep{kipf2017graph} to $2$-SCs.  \citet{glaze2021principled} propose a Simplicial Neural Network for trajectory prediction, studying its permutation and orientation equivariance properties. 
% molecular message passing

\paragraph{Molecular substructures} A few other works have extended GNNs to account for %the presence of 
molecular substructures. Junction Trees (JT), which conveniently represent singletons, bonds and rings as supernodes in a tree, have been used in molecular graph generation~\citep{jin2018junction, jin2019learning}. JTs are also used in the recent work of~\citet{Fey2020_himp}, who employs them to design a hierarchical message passing scheme based on the tree structure. However, this hierarchy has a different configuration than the one cell complexes provide. Information about cycles is also used in GSNs~\citep{bouritsas2020improving} to augment the node features, but the model retains the usual message passing procedure of GNNs. These last two models are of particular relevance to the present work, since they utilise information about chemical rings. It is important to remark that CWNs compare favourably with both of them in all our benchmarks. 

\paragraph{Higher-order GNNs} A related line of work has studied lifting graphs into $k$-dimensional tensor representations that can be processed by provably expressive $k$-GNNs \citep{maron2018invariant, maron2019provably, azizian2021expressive}. With higher values of $k$, these models achieve higher-expressivity, but due to the computational complexity this incurs, values of $k \geq 3$ are of little use in practice. Therefore, unlike CWNs, these models cannot explicitly represent in practice chemical rings of common sizes (e.g. five or six). Furthermore, by being upper-bounded by 3-WL, the 2-GNN models cannot count the number of induced cycles of size greater than four (see Appendix \ref{app:proofs_cellular} for details). In contrast, CWNs can easily count these important chemical substructures through the readout operation it performs on the 2-cells. 

\paragraph{Limitations} 
The main limitations of the model are of computational nature. While the computational complexity of the message passing procedure and its preprocessing step is suitable for molecular and geometric graphs, the number of rings (and more generally simple cycles) in general graphs can be exponential in the number of nodes. In that case, one has to resort to smaller 2-cells like triangles, which can be found efficiently in general graphs. Moreover, one has to typically use weights specific for each dimension of the cell complex, increasing the number of parameters compared to GNNs. However, we have shown that our model can compensate this increase with a reduced number of layers and still achieve state-of-the-art results on some of the molecular benchmarks. 

From a theoretical point of view, this work is concerned only with \emph{regular} cell complexes. Adopting this restriction is useful from multiple perspectives: regular cell complexes are easier to analyse, their combinatorial structure completely describes their topology and convolutions can be defined on them through the Sheaf Laplacian (see Appendix \ref{app:convs}). Nonetheless, some of our theoretical results could be extended to non-regular complexes, which could be obtained by lifting transformations not studied in this work, such as attaching 2-cells to paths in the graph. We leave addressing non-regular complexes and their trade-offs to future developments of this work.   

\paragraph{Societal Impacts} Most of our paper is theoretical in 
nature and we do not see immediate direct negative societal impacts. Within the scope of social network applications, we do not yet have sufficient evidence of performance improvement on related benchmarks to justify obvious adoption in such a domain. In contrast, the empirical performance on molecular benchmarks suggests it may have a positive impact on applications of immediate interest in pharmaceutics, such as drug discovery~\citep{drug_discovery}.

% experimental focus on molecular graphs will have a positive impact on applications of immediate interest such as drug discovery.

\paragraph{Conclusion} We have proposed a provably powerful message passing procedure on cell complexes motivated by a novel colour refinement algorithm to test their isomorphism. This allows us to consider flexible lifting operations on graphs to implement more expressive architectures which benefit from decoupling the computational and input graphs. Our methods show excellent performance on diverse synthetic and real-world molecular benchmarks. 

\section*{Funding and Acknowledgements}

YW and GM acknowledge support from the ERC under the EU's Horizon 2020 %research and innovation 
programme (grant agreement n\textsuperscript{o} 757983). MB is supported in part by ERC Consolidator grant n\textsuperscript{o} 724228 (LEMAN). The authors declare no competing interests. We are also grateful to Ben Day, Gabriele Corso and Nikola Simidjievski for their helpful feedback. We would also like to thank Vijay P. Dwivedi and Chaitanya K. Joshi for clarifying certain aspects of their Benchmarking GNNs \citep{dwivedi2020benchmarkgnns} work, and to Muhammet Balcilar for signalling a numerical precision issue in early SR graphs experiments. 

% \newpage
\bibliographystyle{plainnat}
\bibliography{references}

\newpage 

\appendix

\section{Proofs}
\label{app:proofs_cellular}
\subsection{Cellular WL Results}

In this section, we assume basic familiarity with the WL test and its higher-order variants. For an introduction to these topics, we refer the reader to the survey of \citet{sato2020survey}. We begin by introducing a few useful concepts. 

\begin{definition}
A \textbf{cellular colouring} is a map $c$ that maps a cell complex $X$ and one of its cells $\sigma$ to a colour from a fixed colour palette. We denote this colour by $c^X_\sigma$. 
\end{definition}

\begin{definition}
Let $X, Y$ be two regular cell complexes and $c$ a cellular colouring. We say that $X, Y$ are $\mathbf{c}$\textbf{-similar}, denoted by $c^{X} = c^{Y}$, if the number of cells in $X$ coloured with a given colour equals the number of cells in $Y$ with the same colour. Otherwise, we have $c^{X} \neq c^{Y}$.  
\end{definition}

We emphasise that in this paper we are interested only in colourings $c$ with the property that any two isomorphic cell complexes are $c$-similar. 

\begin{definition}
A cellular colouring $c$ \textbf{refines} a cellular colouring $d$, denoted by $c \sqsubseteq d$, if for all cell complexes $X$ and $Y$
and all $\sigma \in P_X$ and $\tau \in P_Y$, $c^{X}_\sigma = c^{Y}_\tau$ implies $d^{X}_\sigma = d^{Y}_\tau$. Additionally, if $d \sqsubseteq c$, we say the two colourings are equivalent and we represent it by $c \equiv d$. 
\end{definition}

We state the following result from \citet{bodnar2021weisfeiler} about simplicial colourings, which we translate here directly to cell complexes. The proof is however, identical, and we refer the reader to their work for that. 

\begin{proposition}
\label{prop:refine_multiset}
Let $X, Y$ be any regular cellular complexes with $A \subseteq P_X$ and $B \subseteq P_Y$. Consider two cellular colourings $c, d$ such that $c \sqsubseteq d$. If $\ldblbrace d_\sigma^{X} \mid \sigma \in A \rdblbrace \neq \ldblbrace d_\tau^{Y} \mid \tau \in B \rdblbrace$, then $\ldblbrace c_\sigma^{X} \mid \sigma \in A \rdblbrace \neq \ldblbrace c_\tau^{Y} \mid \tau \in B \rdblbrace$.
\end{proposition}

\begin{corollary}
\label{cor:non_iso_colour}
Consider two cellular colourings $c, d$ such that $c \sqsubseteq d$. For all cell complexes $X$ and $Y$, if $d^{X} \neq d^{Y}$, then $c^{X} \neq c^{Y}$. 
\end{corollary}

This last result implies that if $c$ refines $d$, then $c$ can distinguish all the non-isomorphic cell complexes that $d$ can distinguish. We say that the colouring $c$ is at least as powerful as the colouring $d$. 

In contrast to simplicial complexes, cell complexes have a more flexible structure. The main complication compared to the proofs in \citet{bodnar2021weisfeiler} is that cells can have a variable number of lower-dimensional cells on their boundary. It is therefore useful in many proofs, to separate the cells into buckets containing cells with the same boundary size. The following result helps us do that. 

\begin{proposition}
\label{prop:cwl_face_id}
Let $c^t$ be the CWL colouring at iteration $t$. For all cells $\sigma, \tau$ in any cell complexes $X$ and $Y$, if $\vert \gB(\sigma) \vert \neq \vert \gB(\tau) \vert$, then for any $t > 0$ we have $c^t_\sigma \neq c^t_\tau$.  
\end{proposition}

\begin{proof}
If $\sigma$ and $\tau$ have boundaries of different sizes, then $c^1_\gB(\sigma) \neq c^1_\gB(\tau)$, which immediately implies $c^t_\sigma \neq c^t_\tau$ for all $t > 0$. 
\end{proof}

Next, we show that one can drop the co-boundary adjacencies without sacrificing expressive power.

\begin{lemma}
\label{lemma:drop_cofaces}
    CWL with $\mathrm{HASH}\bigl(c_\sigma^t, c_{\gB}^t(\sigma), c_{\da}^t(\sigma), c_{\ua}^t(\sigma)\bigr)$ is as powerful as CWL with the generalised update rule $\mathrm{HASH}\bigl(c_\sigma^t, c_{\gB}^t(\sigma), c_{\gC}^t(\sigma), c_{\da}^t(\sigma), c_{\ua}^t(\sigma)\bigr)$.
\end{lemma}

\begin{proof}
Let $a^t$ denote the colouring produced by CWL using the general version and $b^t$ the colouring produced using the restricted version at iteration $t$. It can be verified that $a^t \sqsubseteq b^t$ because it considers the additional $c_{\gB}^t(\sigma)$ colours in the refinement rule. We  now prove $b^{t+1} \sqsubseteq a^t$ by induction. Note that to take advantage of Proposition \ref{prop:cwl_face_id}, we shift the time-step by one (i.e. we use $b^{t+1}$ as opposed to $b^{t}$). 

The base case holds since $a^0$ assigns the same colour to all the cells. Suppose $b_\sigma^{t+2} = b_\tau^{t+2}$ for any two cells $\sigma$ and $\tau$ from any cell complexes $X$ and $Y$, respectively. Then we know that $b_\sigma^{t+1} = b_\tau^{t+1}, b_\gB^{t+1}(\sigma) = b_\gB^{t+1}(\tau), b_\ua^{t+1}(\sigma) = b_\ua^{t+1}(\tau)$ and $b_\da^{t+1}(\sigma) = b_\da^{t+1}(\tau)$. The goal is to show that this also implies that $b_\gC^{t+1}(\sigma) = b_\gC^{t+1}(\tau)$.

Given $b_\ua^{t+1}(\sigma) = b_\ua^{t+1}(\tau)$, by definition 
$$\ldblbrace b_{\delta_\sigma}^{t+1} \mid (\cdot, b_{\delta_\sigma}^{t+1}) \in b_\ua^{t+1}(\sigma) \rdblbrace
= \ldblbrace b_{\delta_\tau}^{t+1} \mid (\cdot, b_{\delta_\tau}^{t+1}) \in b_\ua^{t+1}(\tau) \rdblbrace.$$ 
By Proposition \ref{prop:cwl_face_id}, cells with different boundary sizes have different colours. Therefore, we can partition these two multi-sets by the size of the cell boundaries, while preserving the equality between these sub-multisets. Therefore, for each $n \in \sN$:  
$$\ldblbrace b_{\delta_\sigma}^{t+1} \mid (\cdot, b_{\delta_\sigma}^{t+1}) \in b_\ua^{t+1}(\sigma) \mathrm{\ and\ } |\gB(\delta_\sigma)| = n \rdblbrace = \ldblbrace b_{\delta_\tau}^{t+1} \mid (\cdot, b_{\delta_\tau}^{t+1}) \in b_\ua^{t+1}(\tau) \mathrm{\ and\ } |\gB(\delta_\tau)| = n \rdblbrace.$$
Let $\gamma$ be an arbitrary cell. Then for each cell $\delta_\gamma \in \gC(\gamma)$, $\gamma$ exchanges messages with all the other boundary cells of $\delta_\gamma$. Therefore, the colour of each $\delta_\gamma$ with $|\gB(\delta_\gamma)| = n$ shows up with a multiplicity of $n - 1$ in the tuples of $b_\ua^{t+1}(\gamma)$. Eliminating $n - 2$ of these repeated colours for all $\delta_\sigma$ and $\delta_\tau$: 
$$\ldblbrace b_{\delta_\sigma}^{t+1} \mid \delta_\sigma \in \gC(\sigma) \mathrm{\ and\ } |\gB(\delta_\sigma)| = n \rdblbrace = \ldblbrace b_{\delta_\tau}^{t+1} \mid  \delta_\tau \in \gC(\tau) \mathrm{\ and\ } |\gB(\delta_\tau)| = n \rdblbrace.$$
Merging these in a single multi-set gives the colours of the co-boundary cells:
$$b_\gC^{t+1}(\sigma) = \ldblbrace b_{\delta_\sigma}^{t+1} \mid \delta_\sigma \in \gC(\sigma) \rdblbrace = \ldblbrace b_{\delta_\tau}^{t+1} \mid  \delta_\tau \in \gC(\tau) \rdblbrace = b_\gC^{t+1}(\tau).$$
By the induction hypothesis, $a_\sigma^{t} = a_\tau^{t}, a_\gB^{t}(\sigma) = a_\gB^{t}(\tau), a_\gC^{t}(\sigma) = a_\gC^{t}(\tau), a_\ua^{t}(\sigma) = a_\ua^{t}(\tau)$ and $a_\da^{t}(\sigma) = a_\da^{t}(\tau)$. This implies $a_\sigma^{t+1} = a_\tau^{t+1}$. 
\end{proof}

The following theorem shows that we can further prune the CWL update rule by removing the colours associated with the lower adjacencies. The structure of the proof is similar to the one in \citet{bodnar2021weisfeiler}, with the main difference being in the proof of Proposition \ref{prop:partition_low_adj}.

\begin{proof}[\textbf{Proof of Theorem~\ref{thm:sparse_cwl}}]
Let $b^t$ denote the colouring of CWL using $\mathrm{HASH}\bigl(b_\sigma^t, b_{\gB}^t(\sigma), b_{\ua}^t(\sigma)\bigr)$ and $a^t$ the colouring of CWL using the rule $\mathrm{HASH}\bigl(a_\sigma^t, a_{\gB}^t(\sigma), a_{\da}^t(\sigma), a_{\ua}^t(\sigma)\bigr)$ from  Lemma~\ref{lemma:drop_cofaces}. Trivially $a^t \sqsubseteq b^t$ because of the additional argument $c_{\da}^t(\sigma)$ in the update rule. We prove $b^{2t+1} \sqsubseteq a^t$ by induction. As before, the addition by one in $2t+1$ is to allow us to apply Proposition \ref{prop:cwl_face_id} in the induction step. The multiplication by $2$ is due to the fact that the information transmitted through the lower adjacencies in one step is propagated in two steps through the boundary adjacencies.  

As before, the base case trivially holds since $a^0$ assigns the same colour to all cells. Suppose $b_\sigma^{2t+3} = b_\tau^{2t+3}$. By unrolling the hash function two steps in time, we obtain $b_\sigma^{2t+1} = b_\tau^{2t+1}$, $b_{\gB}^{2t+1}(\sigma) = b_{\gB}^{2t+1}(\tau)$, and $b_{\ua}^{2t+1}(\sigma) = b_{\ua}^{2t+1}(\tau)$. We need to prove that $b_\da^{2t+1}(\sigma) = b_\da^{2t+1}(\tau)$ also holds. For the sake of contradiction, assume $b_\da^{2t+1}(\sigma) \neq b_\da^{2t+1}(\tau)$. Then there exists a pair of colours $(\sC_0, \sC_1)$ that shows up (without loss of generality) more times in $b_\da^{2t+1}(\sigma)$ than in $b_\da^{2t+1}(\tau)$. For simplicity, we also assume $b_\sigma^{2t+1} \neq \sC_0 \neq b_\tau^{2t+1}$ as this special case can be easily treated separately. 

For all cell complexes $X$ and all cells $\delta$ in $P_X$, consider the collection of multi-sets $A_X$ indexed by $\delta$:
$$A_X(\delta) = \ldblbrace (b_\psi^{2t+1} = \sC_0, b_\delta^{2t+1} = \sC_1) \mid \psi \in \gC(\delta) \rdblbrace.$$
We are interested in the size of these multi-sets for some specific cells $\delta$. To that end, for each cell $\gamma \in P_X$, we define the multi-set:
$$C_X(\gamma) = \ldblbrace \vert A_X(\delta) \vert \mid \delta \in \gB(\gamma) \rdblbrace.$$
We know that $C_X(\sigma) \neq C_Y(\tau)$ since the sum of the elements of $C_X(\sigma)$, which gives the number of tuples $(\sC_0, \sC_1)$ in $b_\da^{2t+1}(\sigma)$, is greater than the sum of the elements of $C_Y(\tau)$, which gives the number of tuples $(\sC_0, \sC_1)$ in $b_\da^{2t+1}(\tau)$. We prove this contradicts our hypothesis that $b_\sigma^{2t+3} = b_\tau^{2t+3}$. 

\begin{proposition}
\label{prop:partition_low_adj} 
For all regular cell complexes $X, Y$ and all $\sigma \in P_X, \tau \in P_Y$, if $C_X(\sigma) \neq C_Y(\tau)$, then $b_\sigma^{2t+3} \neq b_\tau^{2t+3}$. 
\end{proposition}

\begin{proof}
Given a cell complex $X$ and a cell $\delta \in P_X$, consider the cellular colouring $c^X_\delta = \vert A_X(\delta) \vert$. The idea of the proof is to show that $b^{2t+2} \sqsubseteq c$, which allows us to use Proposition \ref{prop:refine_multiset} for the multi-sets $C_X(\sigma)$ and $C_Y(\tau)$. 

Let $\delta_1, \delta_2$ be two arbitrary cells from any regular cell complexes $X, Y$ such that $c^X_{\delta_1} \neq c^Y_{\delta_2}$. Assume without loss of generality that $\vert A_X({\delta_1}) \vert > \vert A_Y({\delta_2}) \vert$. Two cases can be distinguished for this inequality. In the first case, $b^{2t+1}_{\delta_2} \neq \sC_1$, which implies $\vert A_X(\delta_1) \vert > \vert A_Y(\delta_2) \vert = 0$ and, therefore, $b^{2t+1}_{\delta_1} = \sC_1$. Then $b^{2t+2}_{\delta_1} \neq b^{2t+2}_{\delta_2}$. 

In the second case, $b^{2t+1}_{\delta_2} = \sC_1$, which implies $\vert A_X(\delta_1) \vert > \vert A_Y(\delta_2) \vert \geq 0$ and $b^{2t+1}_{\delta_1} = \sC_1$. Then, the difference in the size of the multi-sets is made by the number of times $\sC_0$ shows up in $A_X({\delta_1})$ and $A_Y({\delta_2})$, respectively. By Proposition \ref{prop:cwl_face_id}, all $k$-cells $\gamma$ with $k > 0$ and $b^{2t+1}_\gamma = \sC_0$ must have a fixed boundary size $\vert \gB(\gamma) \vert = n$. Because each cell $\delta \in \gB(\gamma)$ is upper adjacent with every other cell in $\gB(\gamma)$, $b^{2t+1}_\gamma$ appears $n - 1$ times in the tuples inside $b_\ua^{2t+1}(\delta)$. Additionally, note that since the cell complex is regular, self-loops are not allowed and, therefore, $n > 1$. 

Applying this to $\delta_1$ and $\delta_2$, $\sC_0$ shows up $\vert A_X({\delta_1}) \vert \times (n - 1)$ times in $b_\ua^{2t+1}(\delta_1)$ and $\vert A_Y({\delta_2}) \vert \times (n - 1)$ times in $b_\ua^{2t+1}(\delta_2)$. Therefore, $b_\ua^{2t+1}(\delta_1) \neq b_\ua^{2t+1}(\delta_2)$ and, similarly to the first case, $b^{2t+2}_{\delta_1} \neq b^{2t+2}_{\delta_2}$. The results obtained for the two cases prove $b^{2t+2} \sqsubseteq c$.

Applying Proposition \ref{prop:refine_multiset} for the multi-sets $C_X(\sigma)$ and $C_Y(\tau)$, we obtain two non-equal multi-sets:
$$b_\gB^{2t+2}(\sigma) = \ldblbrace b_{\delta_1}^{2t+2} \mid \delta_1 \in \gB(\sigma) \rdblbrace \neq \ldblbrace b_{\delta_2}^{2t+2} \mid \delta_2 \in \gB(\tau) \rdblbrace = b_\gB^{2t+2}(\tau)$$
Since these two multi-sets are used in the colour updating rule, $b_\sigma^{2t+3} \neq b_\tau^{2t+3}$. 
\end{proof}

Therefore, $b_\da^{2t+1}(\sigma) = b_\da^{2t+1}(\tau)$. Finally, applying the induction hypothesis, we have that $a_\sigma^{t} = a_\tau^{t}, a_{\gB}^{t}(\sigma) = a_{\gB}^{t}(\tau), a_{\ua}^{t}(\sigma) = a_{\ua}^{t}(\tau)$ and  $a_\da^t(\sigma) = a_\da^t(\tau)$. Then $a^{t+1}_\sigma = a^{t+1}_\tau$. 
\end{proof}

\begin{proof}[\textbf{Proof of Theorem~\ref{thm:skeleton}}]
Consider the map $f: \gG \to \gX$, a skeleton-preserving lifting transformation from the space of graphs $\gG$, to the space of regular cell complexes $\gX$. Let $g_G: V_G \to P_{f(G)^{(0)}}$ be the graph isomorphism associated to $f$ between the vertices of $G$ and the 0-cells of $f(G)$ for all $G \in \gG$. Let $c^{G, t}$ be the WL colouring of graph $G$ at iteration $t$ and $a^{f(G), t}$ the colouring of $f(G)^{(1)}$ induced by the isomorphism $g_G$ (i.e $a^{f(G)^{(1)}, t}_{g(v)} := c^{G, t}_v$) at the same time step $t$. 

Because $f(G)^{{(1)}}$ and $G$ are isomorphic as graphs and WL is invariant under isomorphism, $a^{f(G)^{(1)}, t} = c^{G, t}$. It follows that for all graphs $G_1, G_2 \in \gG$, if $c^{G_1, t} \neq c^{G_2, t}$ then $a^{f(G_1)^{(1)}, t} \neq ^{f(G_2)^{(1)}, t}$. Let $b^t$ be the CWL colouring of the 0-cells at iteration $t$. The goal is to show that for all regular cell complexes $X, Y \in f(\gG)$, $b^t \sqsubseteq a^t$. By transitivity and combined with Corollary \ref{cor:non_iso_colour}, it follows that if $c^{G_1, t} \neq c^{G_2, t}$, then $b^{f(G_1), t} \neq b^{f(G_2), t}$. 

The base case trivially holds. Let $\sigma, \tau$ be two 0-cells in $X \in f(\gG)$ and $Y \in f(\gG)$, respectively such that $b^{t+1}(\sigma) = b^{t+1}(\tau)$. Since 0-cells have only upper adjacencies, the equality implies that $b^t_\sigma = b^t_\tau$ and $b_\ua^t(\sigma) = b_\ua^t(\tau)$. The latter multi-set equality further implies
$$\ldblbrace b_{\delta_\sigma}^t \mid ( b_{\delta_\sigma}^t, \cdot) \in b_{\ua}(\sigma) \rdblbrace = \ldblbrace b_{\delta_\tau}^t \mid (b_{\delta_\tau}^t, \cdot) \in b_{\ua}(\tau) \rdblbrace.$$
Equivalently, for 0-cells of a cell complex whose 1-skeleton is a graph (i.e. not a multi-graph), this can be rewritten as
$$\ldblbrace b_{\delta_\sigma}^t \mid \delta_\sigma \in \nup(\sigma) \rdblbrace = \ldblbrace b_{\delta_\tau}^t \mid \delta_\tau \in \nup(\tau) \rdblbrace.$$
By the induction hypothesis we have $a^t_{\sigma} = a^t_{\tau}$ and
$$\ldblbrace a_{\delta_\sigma}^t \mid \delta_\sigma \in \nup(\sigma) \rdblbrace = \ldblbrace a_{\delta_\tau}^t \mid \delta_\tau \in \nup(\tau) \rdblbrace.$$
These equalities imply $a^{t+1}(\sigma) = a^{t+1}(\tau)$.
\end{proof}

\begin{figure}[t]
    \centering
    \begin{subfigure}{0.44\textwidth}
        \centering
        \includegraphics[width=\linewidth]{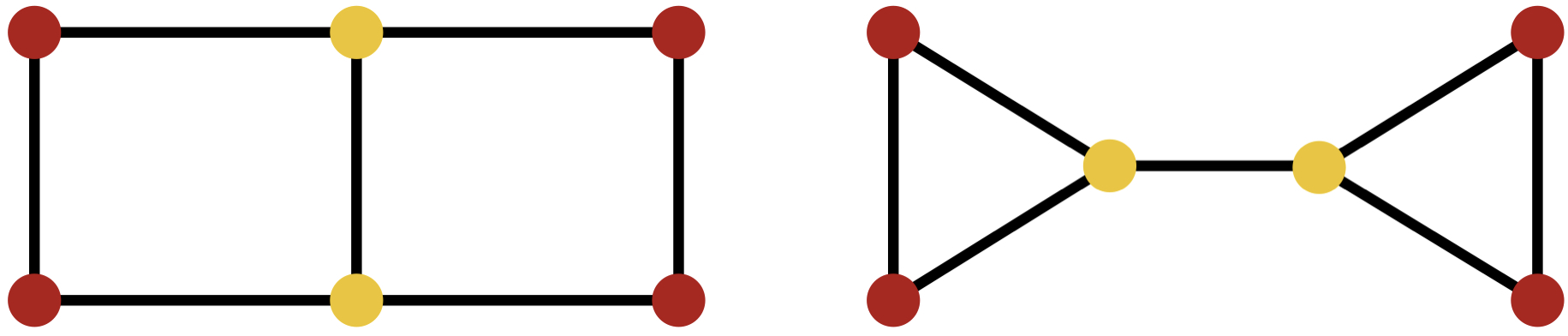}
    \end{subfigure}
    \hfill
    \begin{subfigure}{0.5\textwidth}
        \centering
        \includegraphics[width=\linewidth]{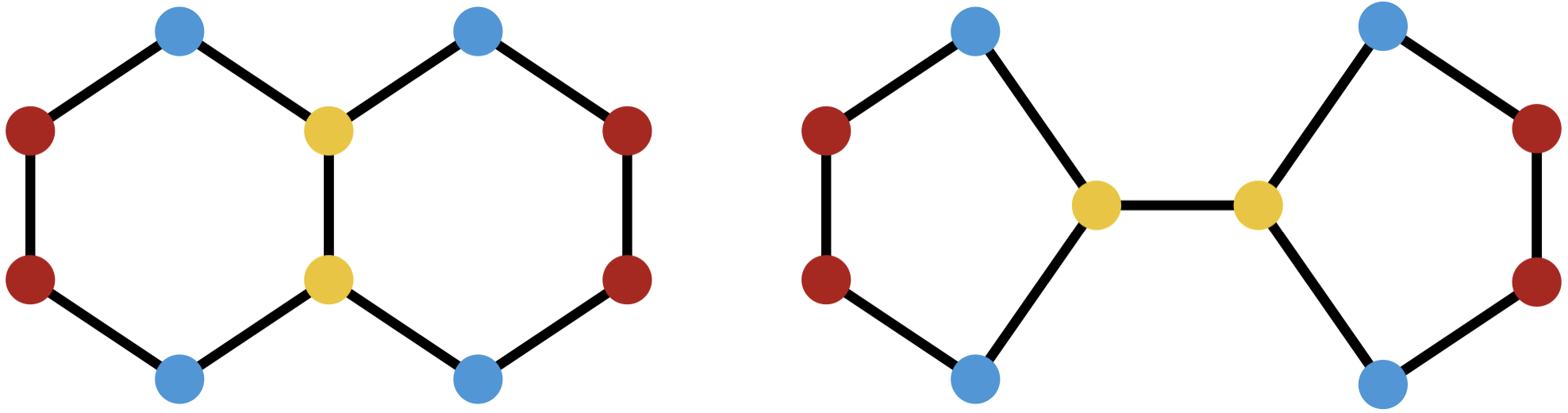}
    \end{subfigure}
    \caption{(Left) A pair of non-isomorphic graphs indistinguishable by WL, but distinguishable by CWL with a clique complex, ring or cycle-based lifting. (Right) A pair of non-isomorphic molecular graphs (Decalin and Bicyclopentyl) indistinguishable by WL but distinguishable by CWL with a ring-based or cycle-based lifting. The node colours show the stable colouring reached by WL. }
    \label{fig:wl_corollary}
\end{figure}

\begin{proof}[\textbf{Proof of Corollary~\ref{cor:WL_lifting_maps}}]
Due to Theorem \ref{thm:skeleton}, it is sufficient to find some examples of non-isomorphic graph pairs that WL cannot distinguish, but CWL can with the given lifting transformations. Figure \ref{fig:wl_corollary} includes such examples. Based on Proposition \ref{prop:cwl_face_id}, CWL can distinguish these graphs since it can count the number of substructures (e.g. triangles, rings, cycles) that the lifting is based on. 
\end{proof}

The next proposition shows that CWL can identify cells that are $n$-simplices. 

\begin{proposition}[Simplex Identification]
\label{prop:cwl_simplex_identification}
Let $X, Y$ be regular cell complexes and $\sigma \in P_X, \tau \in P_Y$ two cells. Denote by $c^t$ the CWL colouring at iteration $t$. Suppose $\sigma$ is an $n$-simplex and $\tau$ is not.  Then $c^t(\sigma) \neq c^t(\tau)$ for all $t \geq n + 1$.
\end{proposition}

\begin{proof}
The base case holds since $c_\sigma^1 \neq c_\tau^1$ if $\sigma$ is a vertex and $\tau$ is a cell of another dimension. This is because $\sigma$ has no boundary adjacencies, while $\tau$ does. 

Suppose the statement holds for $n$-simplices. Then, an $(n+1)$-simplex can be identified by having $n+2$ $n$-simplices on its boundary. By Proposition \ref{prop:cwl_face_id}, the colour of $\sigma$ encodes the boundary size. Furthermore, by the induction hypothesis $c_\gB^{n+1}(\sigma)$ encodes the fact that the boundary cells are $n$-simplices. 
\end{proof}

\begin{proof}[\textbf{Proof of Theorem~\ref{thm:lifting3WL}}]
The sub-results of the theorem can be proven by finding pairs of graphs from the same family of Strongly Regular Graphs that can be distinguished by CWL with the corresponding lifting transformations. Graphs in this family are provably indistinguishable by the higher-order 3-WL test~\citep{bodnar2021weisfeiler}. 

\begin{figure}
    \centering
    \includegraphics[width=0.65\textwidth]{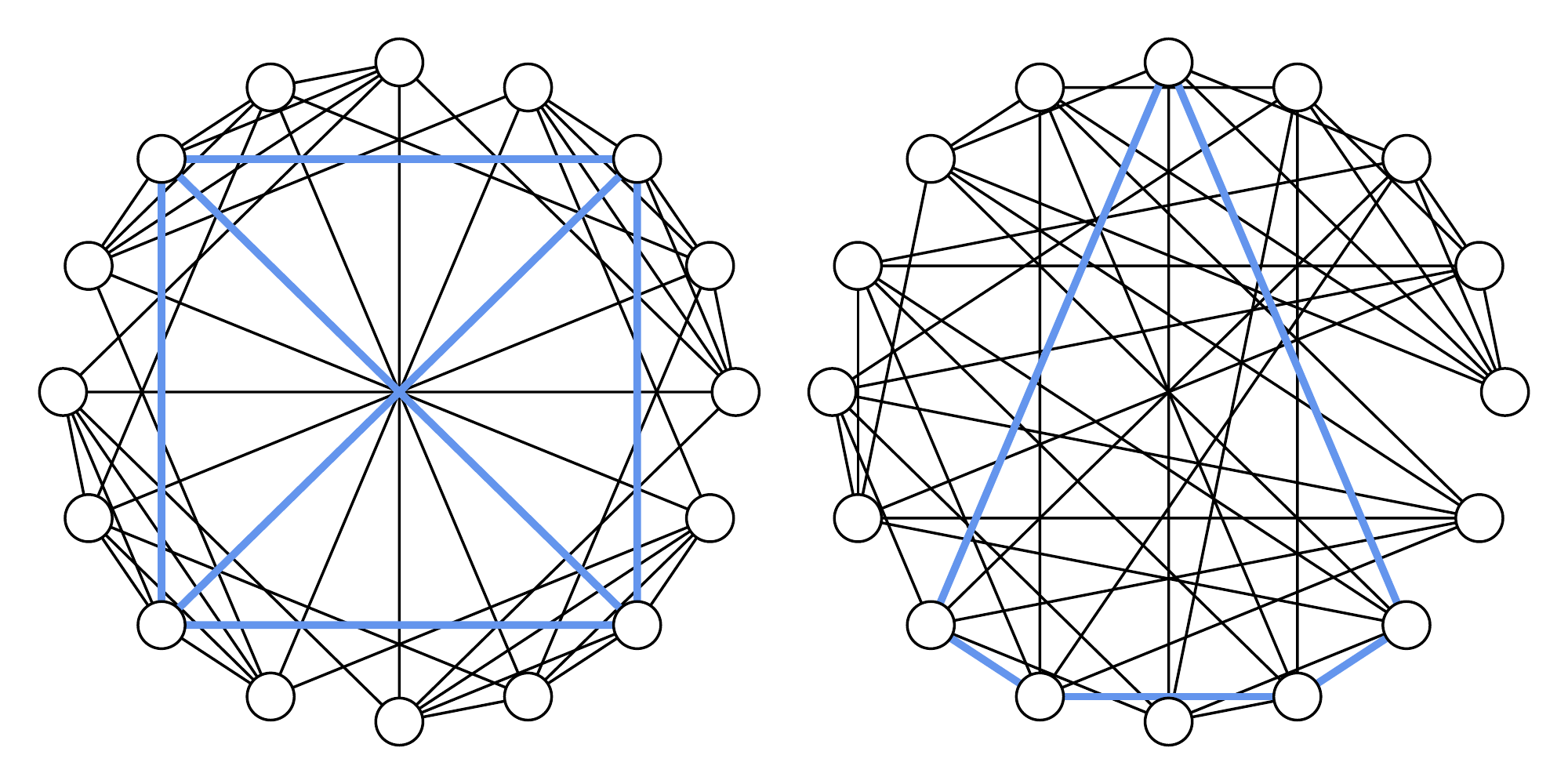}
    \caption{The two SR graphs in family SR$(16,6,2,2)$: Rook's 4$\times$4 (left) and Shrikhande (right). The $3$-WL test is not able to deem them as non-isomorphic. Contrary to the Shrikhande graph, Rook’s graph possesses $4$-cliques. The Shrikhande graph, however, features $5$-rings, not present in Rook’s. Instances of these substructures are marked in blue. With appropriate lifting procedures, CWL can distinguish between them.} 
    \label{fig:sr_pair}
\end{figure}

\paragraph{Ring-based lifting} We can show that there is a pair of SR graphs in the same family with a different number of induced cycles of a certain size. We include such an example in Figure \ref{fig:sr_pair}. The two graphs differ in the number of $4$-, $5$-, $6$- and $8$-rings (see Table~\ref{tab:sr_ring_counts}), which indirectly proves 3-WL cannot count induced cycles of these sizes. It is also natural to conjecture that 3-WL cannot count induced cycles of size strictly larger than $3$. In contrast, CWL($4$-$\mathrm{IC}$) is sufficient to distinguish these two graphs. 

\paragraph{Clique complex lifting} We can leverage on the same example: the graph on the right does not possess $4$-cliques, contrary to the graph on the left (one such example is marked in blue). This proves that 3-WL cannot count cliques of size $4$. As shown by \citet{bodnar2021weisfeiler}, this result immediately implies that SWL (and consequently CWL) with a clique complex lifting is not less powerful than 3-WL.

\begin{table}[h!]
    \centering
    \caption{Number of cycles and induced cycles (rings) on the SR graphs in family SR$(16,6,2,2)$.}
    \label{tab:sr_ring_counts}\vspace{1mm}
    % \resizebox{\columnwidth}{!}{
    \begin{tabular}{l|cccccc}
    \toprule
    Graph $\downarrow$ / Size $\rightarrow$ & 
        $3$ (Tri.) & 
        $4$ &
        $5$ &
        $6$ &
        $7$ &
        $8$ \\
    \midrule
    Rook's 4$\times$4 (cycles)&
        32 &
        60 &
        288 &
        1,248 &
        4,032 &
        11,952 \\
    Shrikhande (cycles)&
        32 &
        60 &
        288 &
        1,248 &
        4,032 &
        11,688 \\
    \midrule
    Rook's 4$\times$4 (rings)&
        32 &
        36 &
        0 &
        96 &
        0 &
        72 \\
    Shrikhande (rings)&
        32 &
        12 &
        96 &
        64 &
        0 &
        36 \\
    \bottomrule
    \end{tabular}
    % }
\end{table}

\paragraph{Cycle-based lifting} To prove the result for this lifting transformation we leverage on a result by \citet{ARVIND202042}, who show that 2-Folklore WL (which is equivalent to 3-WL \citep{sato2020survey}) cannot count subgraph cycles of size strictly larger than 7. Table \ref{tab:sr_ring_counts} illustrates this for the same example as above. Since CWL can count the number of 8-cycles when the lifting transformation $k$-$\mathrm{C}$ with $k \geq 8$ is used (see Proposition~\ref{prop:cwl_face_id}), this proves the result. 

% \paragraph{The union of the transformations} By Proposition \ref{prop:cwl_simplex_identification}, CWL can distinguish $n$-simplices ($(n+1)$-cliques) from general 2-cells. Therefore, CWL can still distinguish the graphs from Figure \ref{fig:sr_pair} based on the number of cliques when using such a lifting transformation.
\end{proof}

We note that while the proof above is purely based on substructure counts, the superior expressive power of CWL is very likely not limited to counting the substructures involved in the lifting transformation. We have seen evidence in favour of this claim in the SR experiment in Section \ref{sec:results}, where message passing layers reduced the failure rate. 

Next, we prove a statement comparing Simplicial WL and CWL. This will later be used to show that CWNs are strictly more powerful than MPSNs when a lifting transformation based on the clique complex and rings is used.  

\begin{definition}
A subset $L$ of a cell complex $X$ is called a \textbf{subcomplex} if it is a union of cells of $X$ containing the closures of these cells.
\end{definition}

\begin{theorem}
Let $f: \gG \to \gX$ be a skeleton-preserving transformation such that for any graph $G$, the clique complex of $G$ is a subcomplex of $f(G)$. Then CWL($f$) is at least as powerful as SWL using the clique complex lifting at distinguishing non-isomorphic graphs. 
\end{theorem}

\begin{proof}

Let $c^t$ be the simplicial colouring performed by SWL. We can extend it into a cellular colouring $a^t$ defined as follows:
\begin{equation}
 a^{X,t}_\sigma := 
    \begin{cases}
      c^{L, t}_\sigma & \mathrm{\ if\ } X_\sigma \subseteq L \\
      \mcirc & \mathrm{otherwise}
    \end{cases}   \nonumber
\end{equation}
where $L$ is the maximal simplicial complex that is a subcomplex of $X$ and $\mcirc$ is a special colour assigned to the cells that are not simplices. Let $h: \gG \to \gX$ be the clique-complex lifting map. Then, it is easy to see that for all graphs $G_1, G_2 \in \gG$, if $c^{h(G_1), t} \neq c^{h(G_2), t}$, then $a^{f(G_1), t} \neq a^{f(G_2), t}$. Let $b^t$ be the CWL colouring map at iteration $t$. We aim to show that $b^{t+n+1} \sqsubseteq a^t$ by using Proposition \ref{prop:cwl_simplex_identification}. Then, by transitivity and using Corollary \ref{prop:refine_multiset}, if $c^{h(G_1), t} \neq c^{h(G_2), t}$, then $b^{f(G_1), t+n+1} \neq b^{f(G_2), t+n+1}$. 

Let $n$ be the maximum dimension of the cells used by the lifting transformation $f$. As usual, the base case holds at initialisation since $a^0$ assigns the same colour to all the cells. Let $\sigma, \tau$ be two cells from the regular cell complexes $X, Y \in f(\gG)$. When $\sigma$ and $\tau$ are not simplices, then $a^{t}_\sigma = a^{t}_\tau = \mcirc$. Suppose $\sigma$ and $\tau$ are both simplices and $b_\sigma^{t+n+2} = b_\tau^{t+n+2}$. Then we know that $b_\sigma^{t+n+1} = b_\tau^{t+n+1}, b_\gB^{t+n+1}(\sigma) = b_\gB^{t+n+1}(\tau)$ and $b_\ua^{t+n+1}(\sigma) = b_\ua^{t+n+1}(\tau)$. Since $\sigma$ and $\tau$ are simplices, their boundary cells are also lower-dimensional simplices, so by induction hypothesis, $a_\gB^{t}(\sigma) = a_\gB^{t}(\tau)$.

Let us consider the equality between the colours involving the upper adjacent cells. By expanding the definition we have: 
\begin{align*}
    \ldblbrace (b_{\delta_1}^{t+n+1}, b_{\delta_2}^{t+n+1}) \mid \delta_1 \in \nup(\sigma), \delta_2 \in \gC(\sigma, \delta_1) \rdblbrace \\ = \ldblbrace (b_{\delta_1}^{t+n+1}, b_{\delta_2}^{t+n+1}) \mid \delta_1 \in \nup(\tau), \delta_2 \in \gC(\tau, \delta_1) \rdblbrace.
\end{align*}
Generally, not all of these adjacencies involve simplices. For instance, a 2-simplex could incident to a general 3-cell. However, by Proposition \ref{prop:cwl_simplex_identification} this equality must still hold if we restrict the multi-sets to the colour of those cells that are simplices: 
\begin{align*}
    \ldblbrace (b_{\delta_1}^{t+n+1}, b_{\delta_2}^{t+n+1}) \mid \delta_1 \in \nup(\sigma), \delta_2 \in \gC(\sigma, \delta_1), \mathrm{\ and\ } \delta_1, \delta_2 \mathrm{\ are\ simplices} \rdblbrace \\ = \ldblbrace (b_{\delta_1}^{t+n+1}, b_{\delta_2}^{t+n+1}) \mid \delta_1 \in \nup(\tau), \delta_2 \in \gC(\tau, \delta_1), \mathrm{\ and\ } \delta_1, \delta_2 \mathrm{\ are\ simplices} \rdblbrace.
\end{align*}
These multi-sets, give exactly the upper adjacencies used by SWL for computing its colouring map $c^t$. Therefore, by the induction hypothesis, $a_\ua^{t}(\sigma) = a_\ua^{t}(\tau)$. Finally, this proves $a_\sigma^{t+1} = a_\tau^{t+1}$. 
\end{proof}

\begin{corollary}
\label{cor:cwl_better_than_swl}
CWL($k_1$-$\mathrm{CL} \cup k_2$-$\mathrm{IC}$) and CWL($k_1$-$\mathrm{CL} \cup k_2$-$\mathrm{C}$) are strictly more powerful than SWL($k_1$-$\mathrm{CL}$) for all $k_2 \geq 5$. 
\end{corollary}

\begin{proof}
The second pair of graphs from Figure \ref{fig:wl_corollary} cannot be distinguished by SWL($k_1$-$\mathrm{CL}$) because it has no cliques greater than two, but it can be distinguished by CWL with the liftings above because of the different number of (induced) cycles. 
\end{proof}

\subsection{CW Network Proof} 

\begin{proof}[\textbf{Proof of Theorem~\ref{thm:CWandCWN}}]
Let $c^t$ denote the colouring of CWL at iteration $t$ and $h^t$ the colouring (i.e. features) produced by a CW-Network as described in Section \ref{sec:CWN_MMP}. Without loss of generality (Theorem \ref{thm:sparse_cwl}), we use only boundary and upper adjacencies for both methods. 

To show CWNs are at most as powerful as CWL, we must show $c^t \sqsubseteq h^t$. Again, we show this by induction. For a CWN with $L$ layers we assume $h^t = h^L$ for all $t > L$.
Let $\sigma, \tau$ be two cells with $c^{t+1}_\sigma = c^{t+1}_\tau$. Then, $c^t_\sigma = c^t_\tau$, $c^t_\gB(\sigma) = c^t_\gB(\tau)$ and $c^t_\ua(\sigma) = c^t_\ua(\tau)$. By the induction hypothesis, $h^t_\sigma = h^t_\tau$, $h^t_\gB(\sigma) = h^t_\gB(\tau)$ and $h^t_\ua(\sigma) = h^t_\ua(\tau)$. 

If $t+1 > L$, then $h^{t+1}_\sigma = h^{t}_\sigma = h^t_\tau = h^{t+1}_\tau$. Otherwise, $h^{t+1}$ is given by Equation \ref{eq:cwn_update} involving the update function $U$, the aggregate function AGG and the message functions $M_\gB, M_\ua$. Given that the inputs passed to these functions are equal for $\sigma$ and $\tau$, $h^{t+1}_\sigma = h^{t+1}_\tau$. 

We now prove that CWNs can be as powerful as CWL. Suppose the aggregation from Equation \ref{eq:cwn_update} is injective and the model is equipped with a number of layers $L$ sufficient to guarantee the convergence of the  colouring. Then, we show that $h^t \sqsubseteq c^t$. 
Let $\sigma, \tau$ be two cells with $h^{t+1}_\sigma = h^{t+1}_\tau$. Then, since the local aggregation is injective $h^t_\sigma = h^t_\tau$, $h^t_\gB(\sigma) = h^t_\gB(\tau)$ and $h^t_\ua(\sigma) = h^t_\ua(\tau)$. By the induction hypothesis, $c^t_\sigma = c^t_\tau$, $c^t_\gB(\sigma) = c^t_\gB(\tau)$ and $c^t_\ua(\sigma) = c^t_\ua(\tau)$. Finally, $c^{t+1}_\sigma = c^{t+1}_\tau$.
\end{proof}

The consequence of this result is that CWNs inherit all the properties of CWL. We summarise these in the following Corollary. 

\begin{corollary}
CWNs have the following properties:
\begin{enumerate}[leftmargin=*]
    \item They are at least as powerful as the WL test when using skeleton-preserving lifting transformations. 
    \item They are strictly more powerful than the WL test when using the lifting maps from Corollary \ref{cor:WL_lifting_maps}.
    \item They are not less powerful than 3-WL when using the lifting transformations from Theorem \ref{thm:lifting3WL}. 
    \item They are at least as powerful as MPSNs using the clique complex lifting \citep{bodnar2021weisfeiler} when using a lifting transformation whose output complexes have the clique complex as a subcomplex. 
    \item They are strictly more powerful than MPSNs when using a transformation attaching cells to cliques and rings/cycles. In particular, CWNs using rings are strictly more powerful than MPSNs using a lifting based on triangles (i.e. 2-simplices), since triangles are rings of size 3. 
\end{enumerate}
\end{corollary}

The latter point regarding triangles is important because \citet{bodnar2021weisfeiler} do not use simplices of dimension higher than two in practice.  

\subsection{Equations for Other Adjacencies}

For completeness, we include in this section the equations for the co-boundary and lower adjacent messages.
\begin{equation*}
  m_{\gC}^{t+1}(\sigma) = \text{AGG}_{\tau \in \gC(\sigma)}\Big(M_{\gC}\big(h_{\sigma}^{t}, h_{\tau}^{t}\big)\Big), \quad
        m_{\downarrow}^{t+1}(\sigma) = \text{AGG}_{\tau \in \ndown(\sigma), \delta \in \gB(\sigma, \tau)}\Big(M_{\da}\big(h_{\sigma}^{t}, h_{\tau}^{t}, h_{\delta}^t\big)\Big).\nonumber
\end{equation*}
Together with the adjacencies described in the main text, the update rule takes the form  
\begin{equation*}
    h_{\sigma}^{t+1} = U\Big(h_{\sigma}^{t}, m_{\gB}^{t}(\sigma), m_{\gC}^{t}(\sigma), m_{\downarrow}^{t+1}(\sigma), m_{\uparrow}^{t+1}(\sigma) \Big).
\end{equation*}
As mentioned before, even though these adjacencies are redundant from a colour refinement perspective when the others are used, they might still be employed in other combinations that preserve the expressive power of the test. Additionally, for certain applications, they might still encode important inductive biases. 

\section{Computational Analysis}
\label{app:complexity}

Let $X$ be a $d$-dimensional regular cell complex. For an arbitrary $p$-cell $\sigma$ with boundary size $k$, the  number of $\uparrow$-messages between the $(p-1)$-cells on its boundary is $2*\binom{k}{2}$ and the number of $\gB$-messages it receives is $k$. Let $B_p$ be the maximum boundary size of a $p$-cell in $X$ and $S_p$ the number of $p$-cells. The computational complexity of our message passing scheme is thus $\gO \big( \sum_{p=1}^{d} B_p S_p + 2*\binom{B_p}{2} S_{p} \big)$. For instance, consider the skeleton-preserving lifting based on induced cycles. There, the dimension of the complex is $d = 2$ and we have $B_0 = 0, B_1 = 2$, and $B_2$ equals the size of the maximum induced cycle considered. For all practical purposes, we can consider $d$ and $B_p$ as fixed constants. Then the complexity can be rewritten as $\Theta \big( \sum_{p=1}^{d} S_p \big)$. This is optimal because the complexity is linear in the size of the cell complex and a linear time is required to read the cell complex.

\begin{table}[h!]
    \centering
    \caption{Wall-clock training and evaluation times on ZINC; mean, std over $10$ runs (seconds).}
    \label{tab:zinc_training_time}
    \vspace{1mm}
    % \resizebox{\columnwidth}{!}{
    \begin{tabular}{l|cccc}
    \toprule
    Model & 
        Training (Epoch) & 
        Eval (Train) &
        Eval (Val) &
        Eval (Test) \\
    \midrule
    GIN &
        4.582 $\pm$ 0.012 &
        3.138 $\pm$ 0.071 & 
        0.310 $\pm$ 0.002 &
        0.309 $\pm$ 0.001 \\
    GIN-small &
        3.737 $\pm$ 0.012 &
        3.070 $\pm$ 0.058 &
        0.304 $\pm$ 0.002 &
        0.303 $\pm$ 0.003 \\
    \midrule
    CIN &
        10.828 $\pm$ 0.059 &
        4.679 $\pm$ 0.051 &
        0.470 $\pm$ 0.002 &
        0.471 $\pm$ 0.003 \\
    CIN-small &
        7.082 $\pm$ 0.041 &
        3.682 $\pm$ 0.056 &
        0.365 $\pm$ 0.002 &
        0.373 $\pm$ 0.030 \\
    \bottomrule
    \end{tabular}
    % }
\end{table}

\begin{table}[h!]
    \centering
    \caption{Wall-clock training and evaluation times on ZINC-FULL; mean, std over $10$ runs (seconds).}
    \label{tab:zincfull_training_time}
    \vspace{1mm}
    % \resizebox{\columnwidth}{!}{
    \begin{tabular}{l|cccc}
    \toprule
    Model & 
        Training (Epoch) & 
        Eval (Train) &
        Eval (Val) &
        Eval (Test) \\
    \midrule
    GIN &
        106.268 $\pm$ 1.991 &
        73.051  $\pm$ 1.742 &
        7.874  $\pm$ 0.174 &
        1.618 $\pm$ 0.039 \\
    GIN-small &
        87.581 $\pm$ 2.343 &
        71.160 $\pm$ 1.865 &
        7.714 $\pm$ 0.206 &
        1.583 $\pm$ 0.037 \\
    \midrule
    CIN &
        249.334 $\pm$ 17.927 &
        107.510 $\pm$ 1.637 &
        11.759 $\pm$ 0.642 &
        2.398 $\pm$ 0.028 \\
    CIN-small &
        163.282 $\pm$ 8.016 &
        85.342 $\pm$ 2.637 &
        9.251 $\pm$ 0.431 &
        1.876 $\pm$ 0.044 \\
    \bottomrule
    \end{tabular}
    % }
\end{table}

In practice, we observed the empirical training runtimes to be contained, even on the largest benchmarks. We performed timing analyses on ZINC and ZINC-FULL, measuring the time required to complete one training epoch and a full performance evaluation on train, validation and test sets. We report the runtimes in Tables~\ref{tab:zinc_training_time} and~\ref{tab:zincfull_training_time} the runtimes measured for our best performing CIN and CIN-small models and by GIN baselines with, approximately, the same number of parameters. We observe that the evaluation runtimes are relatively comparable to those of GIN models and that the difference decreases significantly at inference time (i.e. no backprop). The training runtimes are significantly reduced on CIN-small architectures, which always perform on-par or even better than state-of-the-art baselines, regardless of the imposed parameter budget (see Table~\ref{tab:mol_dataset}). These experiments where run over an NVIDIA\textsuperscript{\textregistered} Tesla V100 GPU device on an Amazon Web Services (AWS) Elastic Cloud (EC) 2 \texttt{p3.16xlarge} instance.

Other than the computational complexity of message passing we need to consider the (one-off) complexity pertaining the graph lifting procedures. Lifting procedures that are more likely to find immediate practical applications involve clique, cycle and induced cycle listing. For cliques, we refer readers to \citet{bodnar2021weisfeiler}, where the authors report theoretical results regarding clique-listing complexity and the practical impact of employing specialised topological data analysis libraries. 

As for cycle-based liftings, specialised cycle-listing algorithms exist. The algorithm in \citet{birmele2013optimal} is able to list all simple cycles in a graph in $ \gO(m + \sum_{c \in C(G)} |c|) $, where $m$ is the number of edges, $C(G)$ is the set of simple cycles in graph $G$ and $|c|$ is the size of the cycle. As for \emph{induced} cycles, the algorithm presented in \citet{ferreira2014amortized} has a listing time of $\tilde{\gO}(m + n C)$, with $n$ and $C$ being the number of nodes and induced cycles, respectively. In certain types of graphs, a better complexity can be obtained. In the case of planar graphs, \citet{chiba1985arboricity} show linear time complexity to list triangles and quadratic complexity for 4-rings. This is very important because almost all molecules are planar in a graph-theoretic sense \citep{SIMMONS1981287} as a direct consequence of the chemical implications of Kuratowski's theorem \citep{kuratowski1930probleme}. However, we are not aware of any improved bounds for finding general induced cycles in planar graphs. Finally, we remind the reader that molecular rings can also be listed from the junction tree representation~\citep{jin2018junction, Fey2020_himp}, obtained by specialised molecular libraries such as RDKit~\citep{Landrum2016RDKit2016_09_4}.

\begin{table}[h!]
    \centering
    \caption{Wall-clock lifting times, mean and std over $5$ runs (seconds).}
    \label{tab:wallclock_lift}
    \vspace{1mm}
    \resizebox{\columnwidth}{!}{
    \begin{tabular}{ l|cccccc}
    \toprule
    Dataset $\downarrow$ / Processes $\rightarrow$ & 
        Seq. & 
        $2$ &
        $4$ &
        $8$ &
        $16$ &
        $32$ \\
    \midrule
    ZINC (12k) &
        320.27 $\pm$ 0.54 &
        169.95 $\pm$ 0.32 & 
        84.90 $\pm$ 0.21 &
        43.38 $\pm$ 0.07 &
        23.17 $\pm$ 0.68 &
        18.59 $\pm$ 0.68 \\
    Mol-HIV (41k) &
        1178.98 $\pm$ 3.90 &
        635.58 $\pm$ 0.83 &
        319.01 $\pm$ 0.40 &
        164.26 $\pm$ 0.52 &
        86.92 $\pm$ 0.77 &
        60.62 $\pm$ 2.05 \\
    ZINC-FULL (250k) &
        6805.35 $\pm$ 16.50 &
        3549.16 $\pm$ 7.73 &
        1782.41 $\pm$ 3.84 &
        918.38 $\pm$ 3.46 &
        492.77 $\pm$ 6.13 &
        383.92 $\pm$ 3.30 \\
    \bottomrule
    \end{tabular}
    }
\end{table}

In our experiments, we implemented a lifting procedure based on the \emph{generic} substructure matching algorithm exposed by the graph-tool Python library, which internally employs VF2~\citep{cordella2004asub} to perform subgraph isomorphism. Noticing that the lifting procedure is embarrassingly parallel w.r.t. the independent graphs in a dataset, we easily parallelised the procedure via Python's Joblib library. On molecular benchmarks we observed the effective time required by preprocessing routines to always be modest compared to the training times. In Table~\ref{tab:wallclock_lift} we report the wall clock runtimes, averaged over $5$ runs, to lift all the graphs in the largest datasets amongst our benchmarks: ZINC, Mol-HIV and ZINC-FULL. The analysis has been conducted considering rings up to size $18$ and by varying the number of parallel processing jobs on a server with an Intel\textsuperscript{\textregistered} Xeon E5-2686 v4 processor with $64$ vCPUs. It is possible to observe that the empirical lifting runtime scales linearly with the number of jobs in the range $[1, 16]$, and that such a simple parallelisation scheme dramatically reduces the preprocessing time on all datasets. When employing $32$ parallel jobs, less than $19$ seconds are required to preproceess the whole ZINC dataset, only $1$ minute is required for Mol-HIV, and we needed slightly more than $6$ minutes to lift all the $250$k graphs in ZINC-FULL. We remark once more that these experiments have been conducted with a \emph{generic} subgraph matching algorithm, and that even more parsimonious computation would be possible by using optimised ring-listing routines.

\section{Symmetries}
\label{app:symmetries}

In line with a recent effort in Geometric Deep Learning to understand different models through the lens of symmetry \citep{Bronstein_etal2017}, we aim here to give a description of the underlying equivariance properties of CW Networks.  

First, let us define the following matrix representation of the boundary relation from Definition \ref{def:boundary_rel}.

\begin{definition}
Let $X$ be a regular cell complex with $S_k$ denoting the number of cells in dimension $k$. The $k$-th unsigned boundary matrix $\mB_k \in \sR^{S_{k-1} \times S_{k}}$ of $X$ is given by $B_k(i,j) = 1$ if $\sigma_i \prec \sigma_j$ and $0$, otherwise. 
\end{definition}

Let $X$ be a regular cell complex of dimension $n$ with boundary matrices $\rmB = (\mB_1, \ldots, \mB_n)$ and feature matrices $\rmX = (\mX_0, \mX_1, \ldots, \mX_n)$ for the cells of different dimensions. Additionally, consider a sequence of permutation matrices $\rmP = (\mP_0, \ldots, \mP_n)$. Denote by $\rmP\rmX = (\mP_0\mX_0, \ldots, \mP_n\mX_n)$ and $\rmP\rmB\rmP^T = (\mP_0\mB_1\mP_1^T, \ldots, \mP_{n-1}\mB_n\mP_n^T)$.

\begin{definition}
\label{def:equiv}
A function $f$ mapping $(\rmX, \rmB) \mapsto \rmX' = (\mX_0', \ldots, \mX_n')$ with the property that $\rmP f(\rmX, \rmB) = f(\rmP\rmX, \rmP\rmB\rmP^T)$ for any $\rmP$ is called cell permutation equivariant. 
\end{definition}

\begin{proof}[\textbf{Proof of Theorem~\ref{thm:CWequivariant}}]
Definition \ref{def:equiv} is similar to the (simplex) permutation equivariance definition from \citet{bodnar2021weisfeiler}, with the subtle difference that the boundary matrices now have a more flexible structure in the case of cell complexes. The high-level idea is to see that all the adjacency matrices used by CWNs (i.e. $\mB_k, \mB_{k+1}, \mB_k^\top\mB_k, \mB_{k+1}\mB_{k+1}^\top$) are permuted accordingly by the permutation matrices in $\rmP$. Therefore, CWNs layers computes the same function up to a permutation of the cells. The proof follows a similar logic to to the one in \citet{bodnar2021weisfeiler} for simplicial networks, and we refer the reader to their work for a detailed proof.   
\end{proof}

It is common in algebraic topology and differential geometry to equip the incidence relation $\sigma \prec \tau$ with additional structure that makes it a signed incidence relation.  

\begin{definition}[\citet{HGh19}]
\label{def:signed_inc}
A signed incidence relation on $P_X$ is a map $[\boldsymbol{\cdot}:\boldsymbol{\cdot}]\colon P_X\times P_X\to \{0,\pm 1\}$ with the properties:
\begin{enumerate}[leftmargin=*]
    \item If $[\sigma : \tau ] \neq 0$, then $\sigma \prec \tau$. 
    \item For any $\sigma \leq \tau$, $\sum_{\gamma \in P_x} [\sigma : \gamma][\gamma : \tau] = 0.$
\end{enumerate}
\end{definition}

This signed incidence relation can be be encoded by the signed incidence (boundary) matrices of $X$. We define these below:

\begin{definition}
Let $X$ be a regular cell complex with a signed incidence relation $[\boldsymbol{\cdot}:\boldsymbol{\cdot}]$. Let $S_k$ denote the number of cells in dimension $k$. The $k$-th signed boundary matrix $\mB_k \in \sR^{S_{k-1} \times S_{k}}$ of $X$ is given by $B_k(i,j) = [\sigma_i : \sigma_j]$. 
\end{definition}

The difference with respect to the unsigned boundary matrices is that the non-zero values of the matrix can be $\pm1$, not just $1$. This can be used to define a notion of orientation equivariance for CW Networks. This ensures that when changing the orientation of the cell complex $X$ (i.e. changing $[\boldsymbol{\cdot}:\boldsymbol{\cdot}]$) one computes the same function up to that change in orientation. 

Let $X$ be a regular cell complex of dimension $n$ described by the \emph{signed} boundary matrices $\rmB = (\mB_1, \ldots, \mB_n)$ and feature matrices $\rmX = (\mX_0, \mX_1, \ldots, \mX_n)$ for the cells of different dimensions. Additionally, consider a sequence of diagonal matrices $\rmT = (\mT_0, \ldots, \mT_n)$ with values in $\pm1$. Additionally, let $\mT_0 = \mI$. Denote by $\rmT\rmX = (\mT_0\mX_0, \ldots, \mT_n\mX_n)$ and $\rmT\rmB\rmT = (\mT_0\mB_1\mT_1, \ldots, \mT_{n-1}\mB_n\mT_n)$.

\begin{definition}
A function $f$ mapping $(\rmX, \rmB) \mapsto \rmX' = (\mX_0', \ldots, \mX_n')$ with the property that $\rmT f(\rmX, \rmB) = f(\rmT\rmX, \rmT\rmB\rmT)$ for any $\rmT$ is called  orientation equivariant. 
\end{definition}

Making CWNs orientation equivariant requires imposing additional constraints on the layers of the model. This proceeds similarly to MPSNs \citep{bodnar2021weisfeiler}. Since applications involving oriented simplicial complexes are out of the scope of this work, we refer the reader to \citet{bodnar2021weisfeiler} for an intuition of how this can be extended to cell complexes. 

\section{Sheaves, Laplacians and Convolutions}\label{app:convs}

It is useful on cell complexes to derive a Laplacian operator based on cellular sheaves \citep{HGh19}, since many interesting Laplacians, such as the (normalised) graph Laplacian \citep{ChGr1997}, the Hodge Laplacian \citep{schaub2020random} and the connection Laplacian \citep{SW12} can be obtained as particular cases. Intuitively, a cellular sheaf is a construction that assigns a vector space to each cell in the complex and a (linear) map for each face relation in the complex $\sigma \leq \tau$. Additionally, these linear maps must satisfy some compositionality constraints imposed by the structure of $P_X$. 

\subsection{Sheaf Laplacian}\label{A:sheaf laplacian}

\begin{definition}
Let $(X,P_X)$ be a regular cell complex, and denote by $\Hilb_K$ the class of Hilbert spaces over a field $K$. A \define{weighted cellular sheaf} $\FF$ is given by the assignment 
\begin{alignat}{2}
\FF\colon P_X&\notag\to \Hilb_K\\
\sigma&\notag\mapsto \FF(\sigma)
\end{alignat}
together with a bounded linear map $\FF_{\sigma\leq \tau}\colon \FF(\sigma)\to \FF(\tau)$ for any $\sigma\leq \tau$.

This data satisfies that $\FF_{\sigma\leq\sigma}=id$ for all $\sigma\in P_X$ and $\FF_{\sigma\leq  \omega}=\FF_{\tau\leq \omega}\circ \FF_{\sigma\leq \tau}$ whenever $\sigma\leq\tau\leq\omega$.
\end{definition}

Given a weighted cellular sheaf $\FF$, we define a chain complex as follows. For each $k=0,1,\dots$ we set 
\[
C^k(X;\FF)=\bigoplus_{dim(\sigma)=k}\FF(\sigma)\, .
\]
Further, we define coboundary maps $\delta^k\colon C^k(X;\FF)\to C^{k+1}(X;\FF)$ by 
\[
\delta^k(x)_\tau=\sum_{dim(\sigma)=k} [\sigma:\tau] \FF_{\sigma\leq \tau}(x_\sigma),
\]
where $[\boldsymbol{\cdot}:\boldsymbol{\cdot}]\colon P_X\times P_X\to \{0,\pm 1\}$ is a signed incidence relation (see Definition \ref{def:signed_inc}).

Given Hilbert spaces $V$ and $W$ and  a bounded linear map $T\colon V\to W$, the adjoint of $T$ is the unique bounded linear map $T^\star\colon W\to V$ satisfying that for all $v\in V$ and all $w\in W$:
\[
%<w,Tv>=<T^\star w,v>\, .
\langle w,Tv \rangle = \langle T^\star w,v \rangle \, .
\]

\begin{definition}
Let $C^\bullet=C^0\to C^1\to \dots $ be a chain complex of Hilbert spaces. The \define{Hodge Laplacian} is the graded linear map defined in degree $k$ as $\Delta^k\colon C^k \notag\to C^k$ with $\Delta^k=(\delta^k)^\star \delta^k+\delta^{k-1}(\delta^{k-1})^\star$. 
When $C^\bullet=C^0\to C^1\to \dots $ is the complex of cochains of a weighted cellular sheaf $\FF$, the Hodge Laplacian is called the \define{sheaf Laplacian} of $X$.
\end{definition}

In particular, the Hodge Laplacian of a cell complex can be obtained by considering the \emph{constant weighted cellular sheaf} with a standard inner product. That is the cellular sheaf where $\gF(\sigma) = \sR$ and the restriction maps $\gF_{\sigma \leq \tau} = \mathrm{id}$. A normalised version of it can also be obtained by carefully adjusting the inner products associated with each $\gF(\sigma)$. This normalisation is always possible for finite cell complexes (see \citet{HGh19} for details). This is very useful because finding normalised versions of Hodge Laplacians is not trivial and even on simplicial complexes \citep{schaub2020random}, the process of constructing one can be quite involved.

\subsection{Convolutional Operators}

One can use the general sheaf Laplacian to define linear, local diffusion operators which, in the GNN literature, are broadly addressed as `convolutional'. Diffusion operators built from the standard graph Laplacian have been employed in several graph neural network architectures \citep{defferrard2016convolutional, kipf2017graph}. Recent works~\citep{ebli2020simplicial, bunch2020simplicial} have introduced convolutional operators on SCs by employing the Hodge Laplacian \citep{schaub2020random}, interpreted as a generalisation of the graph Laplacian. As for cell complexes, here we focus, for simplicity, on the case of a constant sheaf with a standard inner product in $\sR^n$. Then, the matrix representations of $\delta^k$ and $(\delta^k)^*$ are the signed incidence matrices $\mB_k^T$ and $\mB_k$, respectively. Therefore, the Hodge Laplacian can be written in matrix form as  
$$\mL_k = \mB_k^T \mB_k + \mB_{k+1} \mB_{k+1}^T.$$
A convenient way to define a convolutional operator on cochains is by designing a learnable filter parameterised as a polynomial of the Hodge Laplacian. This approach has been already adopted on graphs using the standard graph Laplacian~\citep{defferrard2016convolutional} or more general sheaf Laplacians~\citep{hansen2020sheaf}, and also on simplicial complexes~\citep{ebli2020simplicial}. The advantage of this approach is that of retaining a connection with spectral constructions~\citep{defferrard2016convolutional, ebli2020simplicial} while not requiring any explicit diagonalisation of the operator itself. A polynomial convolutional filter of this kind, when applied to the $p$-cells of a $d$-cell complex, would take the form
\begin{equation}\label{eq:cell_conv}
  H^{t+1} = \psi 
  \Bigl(\sum_{r=0}^{R} \mL_p^r H^{t} W^{t+1}_{r}\Bigr)
  = \psi \Bigl( H^{t} W^{t+1}_{0} + \sum_{r=1}^{R} \mL_p^r H^{t} W^{t+1}_{r} \Bigr).
\end{equation}
\noindent where $H^{t}$ is a matrix gathering $p$-cell representations at layer $t$, $W^{t+1}_{r}$ are learnable parameters, and $\psi$ summarises the application of a bias term and a non-linearity.

\begin{proof}[\textbf{Proof of Theorem \ref{thm:CWequivariant}}] While the structure of the boundary matrices is more flexible in a cell complex than in a simplicial complex, algebraically, the proof is very similar to the proof showing MPSNs generalise simplicial convolutions in \citet{bodnar2021weisfeiler}. We offer here a high-level view of the proof and refer the reader to Appendix C of their paper for a detailed version. 

For a generic $p$-cell $\sigma$, and $r>0$, the application of the $r$-power of the Hodge Laplacian effectively induces an information flow from a generalised notion of $r$-upper and $r$-lower adjacent $p$-cells, i.e. $p$-cells $\tau$ such that there exists a sequence of upper- (respectively, lower-) adjacent $p$-cells $[\gamma_0, \gamma_1, \mathellipsis, \gamma_r]$ such that $\gamma_0 = \sigma, \gamma_r = \tau$.

Therefore, the convolution described above is easily interpreted in terms of a cellular message passing scheme which only exchanges $\uparrow$- and $\downarrow$-messages. Intuitively, the upper- and lower- message functions would share their parameters $W^{t+1}_{r}$ and compute messages by linearly projecting the representations of upper- and lower-adjacent cells (ignoring any information in shared (co)boudaries). Such messages would then be aggregated by summation into an overall message, taken as input by an update function parameterised by $\psi$ and $W^{t+1}_{0}$. A formal derivation of how the equation~\eqref{eq:cell_conv} is rewritten in terms of cellular message passing would closely follow the one provided in~\citet{bodnar2021weisfeiler} for SCs, and we therefore refer readers to Section C of such work. \end{proof}

Normalised versions of the aforementioned Hodge Laplacian can be used to design a model in the spirit of the popular Graph Convolutional Network of~\citet{kipf2017graph}. To this aim, one could resort to normalised sheaves as suggested in~\cite{HGh19}. Additionally, one could explicitly make use of the (co)boundary operators defined in Section~\ref{app:symmetries} to let information flow from lower- and higher- dimensional cells contained in cell (co)boundaries, effectively extending the Simplicial Convolutional Networks recently introduced in~\citet{bunch2020simplicial}. We defer these research directions to future developments of this work.

\section{Experimental details and additional results}\label{app:results}

\subsection{Used Code Assets} 

The model has been implemented in PyTorch~\citep{NEURIPS2019_9015} and by building on top of the PyTorch Geometric library~\citep{fey2019fast}. Lifting operations use the graph-tool\footnote{\url{https://graph-tool.skewed.de/}} Python library and are parallelised via Joblib\footnote{\url{https://joblib.readthedocs.io/en/latest/}}.
PyTorch, NumPy, SciPy and Joblib are made available under the BSD license, Matplotlib under the PSF license, graph-tool under the GNU LGPL v3 license. PyTorch Geometric is made available under the MIT license.

\subsection{Used Computer Resources}
All experiments were run on NVIDIA\textsuperscript{\textregistered} GPUs. Experiments on \textbf{SR}, \textbf{Mol-HIV} and molecular \textbf{TUDatasets} were run on Tesla V100 GPUs with 5,120 CUDA cores and 16GB GPU memory on a \texttt{p3.16xlarge} Amazon Web Services (AWS) Elastic Cloud (EC) 2 instance. Experiments on the social \textbf{TUDatasets} were run on the same GPU devices but with 32GB HBM2 memory mounted on an HPC cluster. All remaining experiments, that is \textbf{CSL}, \textbf{RingTransfer} and \textbf{ZINC}, were run on a machine with TITAN Xp GPUs with 12GB GPU memory and an Intel\textsuperscript{\textregistered} Xeon\textsuperscript{\textregistered} CPU E5-2630 v4 @ 2.20GHz CPU.

\subsection{Model}

In all cases, we apply our model to the $2$-dimensional cell complexes obtained by ring-lifting the original graphs, i.e. we consider nodes and edges as $0$- and $1$-cells, and each induced cycle of size up to $k$ as a $2$-cell. $0$-cell are always endowed with the original node features or learnt node embeddings, if the benchmark prescribes so. The way higher dimensional cells are assigned features depend on the specific benchmark.

Throughout all experiments, we employ cellular message passing layers which update the representation of $p$-cell $\sigma$ as follows:
\begin{align}\label{eq:cin_mp_equations}
    h_{\sigma}^{t+1} &= 
        \mathrm{MLP}^{t}_{U,p} \Big( 
            \mathrm{MLP}^{t}_{\gB,p}\big( (1+\eps_{\gB}) h_{\sigma}^{t} + \sum_{\tau \in \gB(\sigma)} h_{\tau}^{t} \big) \parallel \nonumber \\
    &\quad \qquad \mathrm{MLP}^{t}_{\uparrow,p}\big( (1+\eps_{\uparrow}) h_{\sigma}^{t} + \sum_{\tau \in \gN_{\uparrow}(\sigma), \delta \in \gC(\sigma, \tau)} \mathrm{MLP}^{t}_{M,p} \big( h_{\tau}^{t} \parallel h_{\delta}^t \big) \big) \Big)
\end{align}

Here, $\parallel$ indicates concatenation, $\mathrm{MLP}^{t}_{\gB,p}, \mathrm{MLP}^{t}_{\uparrow,p}$ are 2-Layer Perceptrons and $\mathrm{MLP}^{t}_{U,p}$, $\mathrm{MLP}^{t}_{M,p}$ consist of a dense layer followed by a non-linearity. We neglect messages from cofaces and down-adjacent cells consistently with Theorem~\ref{thm:sparse_cwl}. We name an architecture which stacks $L$ layers of this form as `Cell Isomorphism Network' (CIN). Readout operations are performed as follows. First, for $p \in {0,1,2}$, we compute the joint representation $h_{p}$ of the cells at dimension $p$ by applying a mean or sum readout operation. Then, for complex $\gK$, we compute an overall representation $h_{\gK} = \sum_{p = 0, 1, 2} \mathrm{MLP}_{R,p} \big( h_{p}\big)$, where each $\mathrm{MLP}_{R,p}$ is parameterised as a single dense layer followed by a non-linearity. Complex-wise predictions are obtained by a final dense layer preceded by dropout. All MLP layers internally apply Batch Normalization~\citep{BN} and ReLU activations, unless otherwise specified. All training procedures are performed with the Adam optimiser~\cite{kingma2014adam}.

\subsection{Additional experimental details}

\paragraph{CSL} Each of the $150$ $4$-regular graphs in the CSL dataset comprises $N=48$ nodes and is characterized by \emph{skip number} parameter $C \in \gC = \{2, 3, 4, 5, 6, 9, 11, 12, 13, 16\}$. Parameters $N$ and $C$ determine the isomorphism class $\gG_{N,C}$ of each graph, which we seek to predict. The number of possible classes is $|\gC| = 10$. We employ the same stratified dataset folds in~\citet{dwivedi2020benchmarkgnns}. Consistently with the adopted reference procedure, $0$-cells share the same learnt embedding, while $1$- and $2$-cells are endowed with the sum of the embeddings of the included $0$-cells. As for the optimisation procedure, we set the batch size to $12$ and the initial learning rate of 5\tte-4. which is halved whenever the validation performance does not improve after a patience value of $20$. The training is early stopped as soon as it falls below 1\tte-6, at which step we measure the model test accuracy. The size of hidden layers in our model is set to $160$ and we stack $3$ cellular message passing layers. In this benchmark, we replace Batch Normalisation with Layer Normalization~\cite{ba2016layer}, as the former wsa observed to produce instabilities in the optimisation procedure. At each dimension, cell embeddings are readout via averaging.

\paragraph{SR} These experiments are run in double floating point precision and with untrained models. We initialise the cell complexes associated with SR graph by populating $0$-cells with constant, scalar, unitary signal, and $1$- and $2$-dimensional cells with the sum of the contained $0$-cells. Complexes are embedded in a $16$-dimensional space and, coherently with~\citet{bodnar2021weisfeiler} and~\citet{bouritsas2020improving}, if the $L_2$-distance between the embeddings of two complexes is larger than $\eps = 0.01$, we deem the corresponding graphs to be non-isomorphic. 
We confirmed the validity of the chosen threshold $\eps$ by numerically verifying that, under the described experimental setting, each SR graph in our datasets is deemed isomorphic w.r.t.\ a counterpart obtained by randomly permuting its nodes.
We run a CIN model with $3$ cellular message passing layers, whose hidden layers comprise $16$ units. At each dimension, cell embeddings are readout via summation. As the number of induced cycles of a certain size may be enough to tell apart non-isomorphic SR graphs (see Table~\ref{tab:sr_ring_counts}), an MLP with sum readouts represents a strong baseline, which we additionally run. Such a model applies non-linear dense layers at each cell dimension, and then performs readout operations as in CIN. We set the size of hidden layers to $256$, while the final complex embeddings are embedded in a $16$-dimensional space as in our model. Both approaches are equipped with ELU nonlinearities~\citep{ELU}. 

\paragraph{RingTransfer} This benchmark dataset comprises $5,000$ training graphs. Each graph is randomly associated with one of the $5$ independent labels, which are also assigned as node features to \textbf{source} nodes. Labels are unifomly represented. On this benchmark we run a CIN model with $3$ stacked message passing layers, independently on the ring size. The hidden size of the layers is set to $64$ and we do not apply Batch Normalisation. Differently than in the other benchmarks, we do not need to perform readout operations to compute complex-wise embeddings; instead, we simply take the representation of the $0$-cell corresponding to node \textbf{target} at the last layer of the architecture and use it to predict the label of \textbf{source}. GIN models have always $\lfloor \frac{k}{2} \rfloor$ standard message passing layers with hidden size $64$. The models are trained with an initial learning rate of $10^{-3}$, decayed by a factor of $0.5$ and a patience of $5$ epochs. The training is stopped when the learning rate drops below $10^{-5}$.

\paragraph{TUD} Amongst the datasets from this benchmarking suite: the task in \textbf{MUTAG} is to recognise mutagenic molecular compounds for potentially marketable drug \citep{kazius2005derivation,riesen2008iam}; the one in \textbf{PTC} is to recognise the chemical compounds according to carcinogenicity on rodents \citep{kriege2012subgraph,helma2001predictive}; \textbf{PROTEINS} is about to categorising proteins into enzyme and non-enzyme structures \citep{dobson2003distinguishing,borgwardt2005protein}; \textbf{NCI1} and \textbf{NCI109} deal with identifying chemical compounds against the activity of non-small lung cancer and ovarian cancer cells, respectively \citep{wale2008comparison}; \textbf{REDDIT-BINARY} or \textbf{RDT-B} is a social network dataset where the task is to predict whether a graph belongs to a question-answer-based community or a discussion-based community.
On these datasets, we followed the approach in \citet{GIN}, which prescribes to run a $10$-fold cross-validation procedure and report the maximum of the average validation accuracy across folds. Consistently with such work, we train our model starting from an initial learning rate which is decayed after a fixed amount of epochs and we apply cell-readout operations on the multiscale representations obtained by a Jumping Knowledge scheme~\citep{JK} by performing averaging or summation depending on the dataset, still in accordance with \citet{GIN}. 
We ran a grid-search to tune batch size, hidden dimension, dropout rate, initial learning rate along with its decay steps and strengths, feature initialisation strategy of higher-dimensional cells (mean vs. sum), inclusion of coboundary features in $\uparrow$-messages, number of layers and the dropout position (immediately after readout on cells (``cell read.'') or the final readout on the complex (``comp read.'')). We report the hyperparameter configurations in Table~\ref{tab:tu_hyper}. We finally report that we did not employ Batch Normalization layers in \textbf{RDT-B} since they were observed to produce severe instabilities in the training procedure.

\begin{table}[t]
    \centering
    \caption{Hyperparameter configurations on TUDatasets.}
    \label{tab:tu_hyper}
    \vspace{1mm}
    \resizebox{\columnwidth}{!}{
    \begin{tabular}{l|ccccc|ccc}
        \toprule
        Hyperparameter &
            MUTAG &
            PTC &
            PROTEINS &
            NCI1 &
            NCI109 &
            IMDB-B &
            IMDB-M &
            RDT-B \\
        \midrule
        Batch Size &
            32 &
            32 &
            128 &
            32 &
            32 &
            128 &
            128 &
            32 \\
        Initial LR &
            0.01 &
            0.01 &
            0.01 &
            0.001 &
            0.001 &
            0.001 &
            0.0005 &
            0.001 \\
        LR Dec. Steps &
            20 &
            50 &
            20 &
            20 &
            20 &
            50 &
            20 &
            50 \\
        LR Dec. Strength &
            0.5 &
            0.9 &
            0.5 &
            0.5 &
            0.5 &
            0.5 &
            0.5 &
            0.5 \\
        Hidden Dim. &
            64 &
            16 &
            32 &
            16 &
            64 &
            16 &
            64 &
            64 \\
        Drop. Rate &
            0.5 &
            0.0 &
            0.0 &
            0.5 &
            0.0 &
            0.0 &
            0.5 &
            0.0 \\
        Drop. Pos. &
            cell read. &
            comp read. &
            comp read. &
            comp read. &
            comp read. &
            comp read. &
            comp read. &
            comp read. \\
        Initialisation &
            sum &
            mean &
            mean &
            mean &
            mean &
            mean &
            mean &
            mean \\
        Cobound. in $\uparrow$-msg &
            N &
            N &
            Y &
            Y &
            Y &
            N &
            N &
            N \\
        Num. Layers &
            4 &
            4 &
            3 &
            4 &
            4 &
            4 &
            4 &
            4 \\
        \bottomrule
    \end{tabular}
    }
\end{table}

\paragraph{ZINC} The ZINC benchmarks dataset have been constructed by the ZINC database provided by the Irwin and Shoichet Laboratories in the Department of Pharmaceutical Chemistry at the University of California, San Francisco (UCSF)~\citep{ZINCdataset}. Each graph represents a molecule, with node features indicating the atom type and edge features the type of chemical bond between two atoms. Graph targets correspond to the penalised water-octanol partition coefficient -- logP~\cite{gomez2018automatic}. In these experiments, rings up to size $k=18$ are mapped to $2$-cells, and are assigned feature values as the sum of the learnable atom embeddings for the included $0$-cells (nodes). $1$-cells are assigned learnable bond embeddings if edge-features are considered, otherwise we apply the same policy employed for $2$-cells. We employ the same predefined training, validation and test splits as in~\citet{dwivedi2020benchmarkgnns}, and train our model by minimising the the Mean Absolute Error (MAE) loss on the train targets. As prescribed by the benchmark, the optimisation procedure employs a batch size of $128$ and a dynamic learning rate which starts from $10^{-3}$ and is halved whenever the validation loss does not improve after a patience value we set to $20$. The training is early stopped as soon as it falls below $10^{-5}$. We repeat the training with $10$ different weight initialisations and report the mean of the test MAEs at the time of early stopping. In accordance with the best performing baselines, our CIN model does not use any dropout, and stacks $4$ message passing layers with hidden size $128$. In order to enforce the parameter budget we reduce the size of hidden layers to $48$ and only perform $2$ message passing layers. At each dimension, cell embeddings are readout via summation.

\paragraph{Mol-HIV} This dataset comprises $41127$ molecular graphs associated with a binary label representing their capacity to inhibit HIV replication. The benchmark provides predefined train, validation and test sets based on the ``scaffold splitting'' procedure, which separates molecules based on their two-dimensional structural frameworks~\citep{hu2020open}. As in \textbf{ZINC}, graphs are attributed at the level of nodes and edges, and we directly employ the atom and bond embedding layers provided by the benchmarking platform\footnote{\url{https://github.com/snap-stanford/ogb/blob/master/ogb/graphproppred/mol_encoder.py}} to populate $0$- and $1$-dimensional cells. Rings of size up to $k=6$ are considered as $2$-cells, and are endowed with feature vectors with the same procedure as in \textbf{ZINC}. The value $k=6$ has been chosen from the pool of values $\{ 6, 8, 18 \}$ as it yielded the highest validation performance. The architecture hyperparameters are directly replicated from the HIMP model in~\citet{Fey2020_himp}: $2$ message passing layers, dropout rate of $0.5$ applied after each layer, $64$ as size of hidden layers, constant learning rate of $10^{-4}$, batch size of $128$. We train our model for $150$ epochs. The small CIN model is obtained by simply reducing the size of hidden layers to $48$. At each dimension, cell embeddings are readout via averaging.

\subsection{Ablation study on ZINC}

We end this section by reporting the results of an ablation study we conducted on the ZINC dataset to appreciate the contribution of including rings. In Table~\ref{tab:zinc_abl} we show the average test MAE for two additional CIN models: ``CIN No-Rings small'' and ``CIN No-Rings'', which differ from their original counterparts in that they neglect $2$-cells when performing message passing. In these experiments we always make use of edge features and use the same hyperparameters as our original CIN model. 

\begin{wraptable}{l}{0.4\textwidth}
    \centering
    \vspace{-12pt}
    \begin{minipage}[t]{1.0\linewidth}
        \caption{ZINC Ablation with edge features. The ablation shows the benefits of integrating rings into the message passing procedure.}
        \label{tab:zinc_abl}
        \resizebox{\columnwidth}{!}{
        \begin{tabular}{l  c}
            \toprule
            Method & 
            MAE \\
            \midrule
            
            GatedGCN \citep{bresson2017residual} &
            0.363$\pm$0.009  \\
            
            GIN \citep{GIN}  & 
            0.252$\pm$0.014  \\
            
            PNA \citep{Corso2020_PNA} & 
            0.188$\pm$0.004  \\

            DGN \citep{beaini2020directional} & 
            0.168$\pm$0.003  \\
            
            HIMP \citep{Fey2020_himp} &
            0.151$\pm$0.006 \\
            
            GSN \citep{bouritsas2020improving} & 
            0.108$\pm$0.018  \\
            
            \midrule
            
            GIN-E Custom &
            0.196$\pm$0.007 \\
            
            CIN No-Rings small & 
            0.174$\pm$0.006 \\
            
            CIN No-Rings & 
            0.159$\pm$0.007 \\
            
            CIN-small &
            0.094$\pm$0.004 \\
            
            CIN  & 
            \textbf{0.079}$\pm$\textbf{0.006}  \\
            
            \bottomrule
        \end{tabular}%
        }
    \end{minipage}
    \vspace{-12pt}
\end{wraptable}

In line with our expectations, we observe a decrease in the overall performance of both versions. They are outperformed by the GSN~\citep{bouritsas2020improving} and HIMP~\citep{Fey2020_himp} models, which either include structural information from cycle isomorphism counting (GSN) or additionally perform message passing on the Junction Tree representation of molecules (where rings are considered as nodes). At the same time, we observe ``CIN No-Rings'' still outperforms all other ring-agnostic baselines. We attribute such strong performance to the more natural and richer modelling of edge signals ($1$-cells): this model updates edge representations at each layer as a function of the present representations and those of the incident nodes ($0$-cells). As an additional confirmation of this hypothesis, we implemented an architecture which replicates the same structure as ``CIN No-Rings'', but replaces cellular message passing with GIN-E layers~\citep{hu2020pretraining}. These layers extend the message passing scheme in GIN by accounting for edge features. We refer to this model as ``GIN-E Custom''. Contrary to CIN, it does not update edge representations and performs readout only at the node level. As expected, we observed that ``GIN-E Custom'' is outperformed by all our models, including, in particular, ``CIN No-Rings''.

\end{document}